\pdfoutput=1
\documentclass[11pt]{article}

\usepackage{geometry}
\usepackage{natbib}

\usepackage{pgfplots}
\pgfplotsset{compat=1.17}

\usepackage{microtype}
\usepackage{graphicx}
\usepackage{booktabs}
\usepackage{array}

\usepackage{hyperref}

\usepackage{color}
\usepackage{booktabs}       
\usepackage{amsfonts}       
\usepackage{nicefrac}       
\usepackage{microtype}      
\usepackage{lipsum}
\usepackage{bbm}
\usepackage{dsfont}

\usepackage{pgffor}
\usepackage{bbm}
\usepackage{graphicx}
\usepackage{subcaption}
\usepackage{multirow}
\usepackage[normalem]{ulem}

\usepackage{algorithm}
\usepackage{algorithmic}
\usepackage{amsmath}

\usepackage[title]{appendix}
\usepackage{braket}
\usepackage{complexity}[basic]
\usepackage{stmaryrd}
\usepackage{tikz}

\newcommand{\wfomc}{\ensuremath{\mathsf{WFOMC}}}
\newcommand{\wmc}{\ensuremath{\mathsf{WMC}}}

\newcommand{\iwfomc}{\ensuremath{\mathsf{IncrementalWFOMC}}}

\newcommand{\fotwo}{\ensuremath{\textbf{FO}^2}}
\newcommand{\fothree}{\ensuremath{\textbf{FO}^3}}
\newcommand{\ctwo}{\ensuremath{\textbf{C}^2}}

\newcommand{\vecdelta}{\ensuremath{\boldsymbol{\delta}}}
\newcommand{\Mat}[1]{\ensuremath{\textbf{#1}}}
\newcommand{\BigO}[1]{\ensuremath{\mathcal{O}(#1)}}
\newcommand{\barw}{\ensuremath{\overline{w}}}
\newcommand{\hb}{\ensuremath{\mathsf{HB}}}

\newcommand{\sm}{Sm}
\newcommand{\fr}{Fr}

\newcommand{\Nat}{\mathds{N}}
\newcommand{\Real}{\mathds{R}}

\newclass{\hashP}{\texttt{\#}P}

\newcommand{\comm}[1]{\textcolor{cyan}{}}

\usepackage{amsthm}

\newtheorem{definition}{Definition}

\newtheorem{theorem}{Theorem}
\newtheorem{remark}{Remark}
\newtheorem{lemma}{Lemma}
\newtheorem{corollary}{Corollary}

\theoremstyle{definition}
\newtheorem{example}{Example}[section]

\title{Lifted Inference with Linear Order Axiom}

\author{%
  Jan T\'{o}th \\
   Faculty of Electrical Engineering \\
   Czech Technical University in Prague \\
   Prague, Czech Republic \\
   \and
   Ond\v{r}ej Ku\v{z}elka \\
   Faculty of Electrical Engineering \\
   Czech Technical University in Prague \\
   Prague, Czech Republic \\
}

\date{}

\begin{document}

\maketitle

\begin{abstract}
We consider the task of weighted first-order model counting (\wfomc) used for probabilistic inference in the area of statistical relational learning.
Given a formula $\phi$, domain size $n$ and a pair of weight functions, what is the weighted sum of all models of $\phi$ over a domain of size $n$?
It was shown that computing \wfomc{} of any logical sentence with at most two logical variables can be done in time polynomial in $n$.
However, it was also shown that the task is $\hashP_1$-complete once we add the third variable, which inspired the search for extensions of the two-variable fragment that would still permit a running time polynomial in $n$.
One of such extension is the two-variable fragment with counting quantifiers.
In this paper, we prove that adding a linear order axiom (which forces one of the predicates in $\phi$ to introduce a linear ordering of the domain elements in each model of $\phi$) on top of the counting quantifiers still permits a computation time polynomial in the domain size.
We present a new dynamic programming-based algorithm which can compute \wfomc{} with linear order in time polynomial in $n$, thus proving our primary claim.
\end{abstract}

\section{Introduction}
The task of probabilistic inference is at the core of many statistical machine learning problems and much effort has been invested into performing inference faster.
One of the techniques, aimed mostly at problems from the area of statistical relational learning \citep{SRL/Intro-Book}, being lifted inference \citep{Lifting/Intro-Book}.
A very popular way to perform lifted inference is to encode the particular problem as an instance of the weighted first-order model counting (\wfomc) task.
It is worth noting that applications of \wfomc{} range much wider, making it an interesting research subject in its own right.
For instance, it was used to aid in conjecturing recursive formulas in enumerative combinatorics \citep{Lifting/P-Recursions}.

Computing \wfomc{} in the two-variable fragment of first-order logic (denoted as \fotwo) can be done in time polynomial in the domain size, which is also referred to as \fotwo{} being \textit{domain-liftable} \citep{Lifting/FO2-domain-liftable}.
Unfortunately, it was also shown that the same does not hold in \fothree{} where the problem turns out to be $\hashP{}_1$-complete in general \citep{Lifting/FO3-intractable}.
That has inspired a search for extensions of \fotwo{} that would still be domain-liftable.

Several new classes have been identified since then.
\citet{Lifting/New-liftable-classes} introduced the classes $\textbf{S}^2\fotwo$ and $\textbf{S}^2\textbf{RU}$.
\citet{Lifting/Functionality-Constraint} extended the two-variable fragment with one functionality axiom and showed such language to still be domain-liftable.
That result was later generalized to the two-variable fragment with counting quantifiers (\ctwo) \citep{Lifting/C2-domain-liftable}.
Moreover, \citet{Lifting/C2-Tree-domain-liftable} proved that \ctwo{} extended by the tree axiom is still domain-liftable as well.%
\footnote{Other recent works in lifted inference not directly related to our work presented here are \citep{Lifting/FO2-cell-graphs}, \citep{Lifting/Closed-Form-Formula} and \citep{Lifting/Sampling}.}

Another extension of \ctwo{} can be obtained by adding a linear order axiom.
Linear order axiom \citep{Logic/linear-order} enforces some relation in the language to introduce a linear (total) ordering on the domain elements.
Such a constraint is inexpressible using only two variables, requiring special treatment.
This logic fragment has also received some attention from logicians \citep{Logic/C2}.

In this paper, we show that extending \ctwo{} with a linear order axiom yields another domain-liftable language.
We present a new dynamic programming-based algorithm for computing \wfomc{} in \ctwo{} with linear order.
The algorithm's running time is polynomial in the domain size meaning that \ctwo{} with linear order is domain-liftable.

Even though our result is mostly of theoretical interest, we still provide some interesting applications and experiments.
Among others, we perform exact inference in a Markov Logic Network \citep{MLNs} on a random graph model similar to the one of Watts and Strogatz \citep{Watts-Strogatz}.

\section{Background}
Let us now review necessary concepts, definitions and assumptions as well as notation.

We use boldface letters such as $\Mat{k}$ to differentiate vectors from scalar values such as $n$.
If we do not name individual vector components such as $\Mat{k} = (k_1, k_2,\ldots,k_d)$, then the $i$-th element of $\Mat{k}$ is denoted by $(\Mat{k})_i$.
Since our vectors only have non-negative entries, the sum of vector elements, i.e., $\sum_{i=1}^d (\Mat{k})_i$, always coincides with the $L^1$-norm.
Hence, we use $|\Mat{k}|$ as a shorthand for the sum.
We also introduce special name $\vecdelta_j$ for a vector such that
$$(\vecdelta_j)_i =
    \begin{cases}
        1 \text{ if } i = j,\\
        0 \text{ otherwise}.
    \end{cases}
$$

For a vector $\Mat{k} = (k_1, k_2,\ldots,k_d)$ with $|\Mat{k}|=n$,
$$\binom{|\Mat{k}|}{\Mat{k}}=\binom{n}{k_1,k_2,\ldots,k_d}$$
denotes the multinomial coefficient.
We make use of one non-trivial identity of multinomial coefficients \citep{Combinatorics/Book}, namely
$$\sum_{j=1}^d \binom{n-1}{\Mat{k}-\vecdelta_j}=\binom{n}{\Mat{k}}.$$

We also assume the set of natural numbers $\Nat$ to contain zero and that $0^0=1$.
We use $[n]$ to denote the set $\Set{1,2,\ldots,n}$.

\subsection{First-Order Logic}
We work with a function-free subset of first-order logic.
The language is defined by a finite set of \textit{constants} $\Delta$, a finite set of \textit{variables} $\mathcal{V}$ and a finite set of \textit{predicates} $\mathcal{P}$.
If the arity of a predicate $P \in \mathcal{P}$ is $k$, we also write $P/k$.
An \textit{atom} has the form $P(t_1, t_2, \ldots, t_k)$ where $P/k \in \mathcal{P}$ and $t_i \in \Delta$ $\cup$ $\mathcal{V}$.
A \textit{literal} is an atom or its negation.
A \textit{formula} is an atom and a literal.
More complex formulas may be formed from existing formulas by logical connectives, or by surrounding them with a universal ($\forall x$) or an existential ($\exists x$) quantifier where $x \in \mathcal{V}$.
A variable $x$ in a formula is called $free$ if the formula contains no quantification over $x$.
A formula is called a $sentence$ if it contains no free variables.
A formula is called \textit{ground} if it contains no variables.

As is customary in computer science, we adopt the \textit{Herbrand semantics} \citep{Logic/Herbrand} with a finite domain.
Since we have a finite domain with a one-to-one correspondence to the constant symbols, we denote the domain also with $\Delta$.
We denote the \textit{Herbrand base}  by \hb.
We use $\omega$ to denote a \textit{possible world}, i.e., any subset of \hb.
When we wish to restrict a possible world $\omega$ to only atoms with a particular predicate $P$, we write $\omega[P]$.

We work with logical sentences containing at most two variables (the language of \fotwo{}).
We assume our $\fotwo{}$ sentences to be constant-free.
Dealing with constants in lifted inference is a challenge in its own right.
Treatment of conditioning on evidence as well as using constants in sentences is available in other literature \citep{Lifting/liftability-with-evidence, Lifting/liftability-with-constants}.

\subsection{Weighted Model Counting and Lifted Formulation}
Throughout this paper, we study the \textit{weighted first-order model counting}.
We will also make use of its propositional variant, the \textit{weighted model counting}.
Let us formally define both these tasks.

\begin{definition}{(Weighted Model Counting)}
Let $\phi$ be a logical formula over some propositional language $\mathcal{L}$.
Let \hb{} denote the Hebrand base of $\mathcal{L}$ (i.e., the set of all propositional variables).
Let $w: \hb \mapsto \mathds{R}$ and $\overline{w}: \hb \mapsto \mathds{R}$ be a pair of \textit{weightings} assigning a \textit{positive} and a \textit{negative} weight to each variable in $\mathcal{L}$.
We define
\begin{equation*}
    \wmc(\phi, w, \overline{w}) = \sum_{\omega \subseteq \hb:\omega \models \phi} \prod_{l \in \omega} w(l) \prod_{l \in \hb \setminus \omega} \overline{w}(l).
\end{equation*}
\end{definition}

\begin{definition}{(Weighted First-Order Model Counting)}
Let $\phi$ be a logical formula over some relational language $\mathcal{L}$.
Let $n$ be the domain size.
Let \hb{} denote the Hebrand base of $\mathcal{L}$ over the domain $\Delta = \Set{1,2,\ldots,n}$.
Let $\mathcal{P}$ be the set of the predicates of the language $\mathcal{L}$ and
let $\mathsf{pred}: \hb \mapsto \mathcal{P}$ map each atom to its corresponding predicate symbol.
Let $w: \mathcal{P} \mapsto \mathds{R}$ and $\overline{w}: \mathcal{P} \mapsto \mathds{R}$ be a pair of \textit{weightings} assigning a \textit{positive} and a \textit{negative} weight to each predicate in $\mathcal{L}$.
We define
\begin{align*}
    \wfomc(\phi, n, w, \overline{w}) =
    \sum_{\omega \subseteq \hb:\omega \models \phi} \prod_{l \in \omega} w(\mathsf{pred}(l)) \prod_{l \in \hb \setminus \omega} \overline{w}(\mathsf{pred}(l)).
\end{align*}
\end{definition}

\begin{remark}
Since for any domain $\Delta$ of size $n$, we can define a bijective mapping $\pi$ such that $\pi(\Delta) = \Set{1,2,\ldots,n}$, \wfomc{} is defined for an arbitrary domain of size $n$.
\end{remark}

\subsection{Cells and Domain-Liftability of the Two-Variable Fragment}
We will not build on the original proof of domain-liftability of \fotwo{} \citep{Lifting/FO2-domain-liftable,Lifting/Skolemization}, but rather on the more recent one \citep{Lifting/FO3-intractable}.
Let us review some parts of that proof as we make use of them later in the paper.

An important concept is the one of a \textit{cell}.

\begin{definition}
A cell of a first-order formula $\phi$ is a maximally
consistent set of literals formed from atoms in $\phi$ using only
a single variable.
\end{definition}

\noindent We will denote cells as $C_1(x), C_2(x), \dots, C_p(x)$ and assume that they are ordered (indexed).
Note, however, that the ordering is purely arbitrary.

\begin{example}
Consider $\phi = \sm(x) \wedge \fr(x, y) \Rightarrow \sm(y)$. 

Then there are four cells:
\begin{align*}
C_1(x) &= \sm(x) \wedge \fr(x,x), \\
C_2(x) &= \neg \sm(x) \wedge \fr(x,x),\\
C_3(x) &= \neg \sm(x) \wedge \neg \fr(x,x),\\
C_4(x) &= \sm(x) \wedge \neg \fr(x,x).
\end{align*}

\end{example}

\noindent It turns out, that if we fix a particular assignment of domain elements to the cells and if we then condition on such evidence, the \wfomc{} computation decomposes into mutually independent and symmetric parts, simplifying the computation significantly.

When we say assignment of domain elements to cells, we mean a domain partitioning allowing empty partitions, that is ordered with respect to a chosen cell ordering.
Each partition $S_j$ then holds the constants assigned to the cell $C_j$.
Such partitioning can be captured by a vector.
We call such a vector a \textit{partitioning vector} and often shorten the term to a \textit{p-vector}.

\begin{definition}
Let $C_1,C_2,\dots,C_p$ be cells of some logical formula.
Let $n$ be the number of elements in a domain.
A partitioning vector (or a p-vector) of order $n$ is any vector $\Mat{k} \in \Nat^p$ such that $|\Mat{k}|=n$.
\end{definition}

\noindent Moreover, conditioning on some cells may immediately lead to an unsatisfiable formula.
To avoid unnecessary computation with such cells, we only work with \textit{valid cells} \citep{Lifting/FO2-cell-graphs}.

\begin{definition}
A valid cell of a first-order formula $\phi(x,y)$ is a cell of $\phi(x, y)$ and is also a model of $\phi(x, x)$.
\end{definition}

\begin{example}
Consider $\phi = F(x, y) \wedge (G(x) \vee H(x))$.

Cells setting both $G(x)$ and $H(x)$ to false are not valid cells of $\phi$.
\end{example}

\noindent Let us now introduce some notation for conditioning on particular (valid) cells.
Denote
\begin{align*}
    \psi_{ij}(x,y) &= \psi(x,y) \wedge \psi(y,x) \wedge C_i(x) \wedge C_j(y), \\
    \psi_{k}(x) &= \psi(x,x) \wedge C_k(x),
\end{align*}
and define
\begin{align}
    r_{ij} &= \wmc(\psi_{ij}(A, B), w', \overline{w}'), \\
    w_{k} &= \wmc(\psi_{k}(A), w, \overline{w}),
\end{align}
where $A,B\in\Delta$ and the weights $w'$, $\overline{w}'$ are the same as $w$, $\overline{w}$ except for the atoms appearing in the cells conditioned on.
Their weights are set to one, since their weights are already accounted for in the $w_k$ terms.

Finally, we can write
\begin{align}
\label{eq:2wfomc}
\wfomc(\phi, n, w, \overline{w}) =
\sum_{\Mat{k}\in\Nat^p:|\Mat{k}|=n} \binom{n}{\Mat{k}}&\prod_{i,j\in[p]:i<j}r_{ij}^{(\Mat{k})_i(\Mat{k})_j}\prod_{i\in[p]}r_{ii}^{\binom{(\Mat{k})_i}{2}}w_i^{(\Mat{k})_i},
\end{align}
which implies that universally quantified \fotwo{} is domain-liftable since Equation \ref{eq:2wfomc} may be evaluated in time polynomial in $n$.
Using a specialized \textit{skolemization} procedure for \wfomc{} \citep{Lifting/Skolemization}, we can easily extend the result to the entire \fotwo{} fragment.

\subsection{Cardinality Constraints and Counting Quantifiers}
\wfomc{} can be further generalized to \textit{\wfomc{} under cardinality constraints} \citep{Lifting/C2-domain-liftable}.
For a predicate $P \in \mathcal{P}$, we may extend the input formula by one or more cardinality constraints of the type $(|P|\bowtie k)$, where $\bowtie \in \Set{\leq, =, \geq}$ and $k \in \Nat$.
Intuitively, a cardinality constraint $(|P| = k)$ is satisfied in $\omega$ if there are exactly $k$ ground atoms with predicate $P$ in $\omega$.
Similarly for the inequality signs.

Counting quantifiers are a generalization of the traditional existential quantifier.
For a variable $x \in \mathcal{V}$, we allow usage of a quantifier of the form $\exists^{\bowtie k}x$, where $\bowtie \in \Set{\leq, =, \geq}$ and $k \in \Nat$.
Satisfaction of formulas with counting quantifiers is defined naturally, in a similar manner to the satisfaction of cardinality constraints.
For example, $\exists^{=k}x:\psi(x)$ is satisfied in $\omega$ if there are exactly $k$ constants $\Set{A_1, A_2, \ldots, A_k} \subseteq \Delta$ such that $\forall i \in [k]: \omega \models \psi(A_i)$.

\citet{Lifting/C2-domain-liftable} showed \ctwo{} to be a domain-liftable language.
That was done by reducing \wfomc{} in \ctwo{} to \wfomc{} in \fotwo{} under cardinality constraints and showing that the two-variable fragment with cardinality constraints is also domain-liftable.

\subsection{Linear Order Axiom}
Assuming logic with equality, we can encode that the predicate $R$ enforces a linear ordering on the domain using the following logical sentences \citep{Logic/linear-order}:
\begin{enumerate}
    \item $\forall x: R(x, x)$,
    \item $\forall x \forall y : R(x, y) \vee R(y, x)$,
    \item $\forall x \forall y : R(x,y) \wedge R(y,x) \Rightarrow (x=y)$,
    \item $\forall x\forall y \forall z : R(x,y) \wedge R(y,z) \Rightarrow R(x,z)$.
\end{enumerate}
The last sentence, expressing transitivity of the relation $R$, is the problematic one as it requires three logical variables.
Hence, we will not simply append this axiomatic definition to the input formula but rather make use of a specialized algorithm.
However, we must keep the axioms in mind, when constructing cells.
Substituting $x$ for both $y$ and $z$ into the axioms above leaves us with (after simplification) a single sentence enforcing reflexivity, i.e., $\forall x: R(x, x)$. Only cells adhering to this constraint can be valid.

Throughout this paper, we denote the constraint that a predicate $R$ introduces a linear order on the domain as $Linear(R)$.
For easier readability, we also make use of the traditional symbol $\leq$ for the linear order predicate whenever possible.
We also prefer the infix notation rather than the prefix one as it is more commonly used together with $\leq$ sign.
We also use $(A<B)$ as a shorthand for $(A\leq B) \wedge \neg(B \leq A)$.

We often write $\phi = \psi \wedge Linear(\leq)$, where we assume $\psi$ to be some logical sentence in \fotwo{} or \ctwo{} and $\leq$ one of the predicates of the language of $\psi$.
Let us formalize the model of such a sentence.
\begin{definition}
Let $\psi$ be a logical sentence possibly containing binary predicate $\leq$.
A possible world $\omega$ is a model of $\phi = \psi \wedge Linear(\leq)$ if and only if $\omega$ is a model of $\psi$, and $\omega[{\leq}]$ satisfies the linear order axioms.
\end{definition}

Our usual goal will be to compute \wfomc{} of $\phi$ over some domain.
In such cases, part of the input will be weightings $(w, \barw)$.
Since we are treating $\leq$ as a special predicate that is only supposed to enforce an ordering of domain elements in the models of $\phi$, we will always assume $w(\leq)=\overline{w}(\leq)=1$.

One more consideration should be given to our assumption of having equality in the language.
That is not a hard requirement since encoding equality in \ctwo{} (or \fotwo{} with cardinality constraints) is relatively simple, compared to full first-order logic.
For example, we may use the axioms:
\begin{enumerate}
    \item $\forall x: (x = x)$,
    \item $\forall x \exists^{=1} y : (x = y).$
\end{enumerate}

\begin{example}
As a simple example of what the linear order axiom allows us to express, consider the sentence $\phi = \forall x \forall y: \psi(x, y) \wedge Linear(\leq)$, where
$$\psi(x, y) = T(x) \wedge (x \leq y) \Rightarrow T(y).$$

How can we interpret models of $\phi$?
Due to $Linear(\leq)$, the $\leq$ predicate will define a total ordering on the domain, e.g., $1 \leq 2 \leq \ldots \leq n$.
Thus, we can think of the domain as a sequence.

The formula $\psi(x, y)$ then seeks to split that sequence into its beginning (\textit{head} of the sequence) and its end (\textit{tail} of the sequence).
The predicate $T/1$ denotes the tail of the sequence.
Whenever there is a constant, for which $T/1$ is set to true in a model (it is part of the tail), then all constants \textit{greater} also have $T/1$ set to true.
Constants, for which $T/1$ is set to false, then belong to the sequence head.
\end{example}

\section{Approach}
To prove our main result, we proceed as follows.
First, we present a new algorithm based on dynamic programming that computes \wfomc{} of a universally quantified \fotwo{} sentence in an incremental manner, and it does so in time polynomial in the domain size.
Note, that the assumption of universal quantification is not a limiting one, since we can apply the skolemization for \wfomc{} to our input sentence before running the algorithm.
Second, we show how to adapt the algorithm to compute \wfomc{} of a formula $\phi = \psi \wedge Linear(\leq)$, where $\psi$ is a universally quantified \fotwo{} sentence.
And third, we use the algorithm as a new \wfomc{} oracle in the reductions of \wfomc{} in \ctwo{} to \wfomc{} in \fotwo{}, thus proving \ctwo{} extended by a linear order axiom to be domain-liftable.

\subsection{New Algorithm}
Our algorithm for computing $\wfomc(\phi,n,w,\overline{w})$ for an \fotwo{} sentence $\phi$ works in an incremental manner.
The domain size is inductively enlarged in a similar way as in the \textit{domain recursion rule} \citep{Lifting/FO2-domain-liftable, Lifting/New-liftable-classes}.
For each domain size $i$, the \wfomc{} for each possible p-vector is computed.
The results are tracked in a table $T_i$ which maps possible p-vectors to real numbers (the weighted counts).
The results are then reused to compute entries in the table $T_{i+1}$.
See Algorithm \ref{algo:dp} for details.

\begin{algorithm}[tbh]
\caption{\iwfomc{}}
\label{algo:dp}
\textbf{Input}: An \fotwo{} sentence $\phi$, $n\in\Nat$, weightings $(w,\overline{w})$

\textbf{Output}: $\wfomc(\phi, n, w, \overline{w})$
\begin{algorithmic}[1]
\REQUIRE $\forall i \in \left[n\right] \forall \textbf{k} \in \Nat^p, |\textbf{k}|=i : T_i[\textbf{k}] = 0$
\FORALL{cell $C_j$}
    \STATE $T_1[\vecdelta_j] = w_j$
\ENDFOR
\FOR{$i=2$ to $n$}
    \FORALL{cell $C_j$}
        \FORALL{$(\textbf{k}_{old}, W_{old}) \in T_{i-1}$}
            \STATE $W_{new} \gets W_{old} \cdot w_j \cdot \prod_{l=1}^p r_{jl}^{(\textbf{k}_{old})_l}$
            \STATE $\textbf{k}_{new} \gets \textbf{k}_{old} + \vecdelta_j$
            \STATE $T_i[\textbf{k}_{new}] \gets T_i[\textbf{k}_{new}] + W_{new}$
        \ENDFOR
    \ENDFOR
\ENDFOR
\RETURN $\sum_{\textbf{k} \in \mathds{N}^p:|\textbf{k}|=n} T_n[\textbf{k}]$
\end{algorithmic}
\end{algorithm}

To compute an entry $T_{i+1}[\Mat{u}]$ for a p-vector $\Mat{u}$, we must find all entries $T_i[\Mat{k}]$ such that $\Mat{k} + \vecdelta_j = \Mat{u}$ and $C_j$ is one of the cells.
Intuitively speaking, we will assign the new domain element $(i+1)$ to the cell $C_j$, which will extend the existing models with new ground atoms containing the new domain element.
The models will be extended by atoms corresponding to the subformula $\psi_j(i+1)$ (which, if we are only working with valid cells, are simply the positive literals from $C_j$) and by atoms corresponding to the subformula $\psi_{jk}(i+1, i')$ for each cell $C_k$ and each domain element already processed (i.e., $1 \leq i' < i+1$).
As we can construct the new models by extending the old, we can also compute the new model weight from the old.
The weight update can be seen on Line 7 of Algorithm \ref{algo:dp}.

To prove correctness of Algorithm \ref{algo:dp}, we prove that its result is the same as is specified in Equation \ref{eq:2wfomc}.
For better readability, we split the proof into an auxiliary lemma, which proves a particular property of table entries at the end of each iteration $i$, and the actual statement of the algorithm's correctness.

\begin{lemma}
\label{l:induction}
At the end of iteration $i$ of the for-loop on lines $4-12$, it holds that
\begin{align*}
    T_{i}[\textbf{k}] = \binom{i}{\textbf{k}} \prod_{i,j\in[p]:i<j} r_{ij}^{(\textbf{k})_i(\textbf{k})_j} \cdot \prod_{i=1}^{p}r_{ii}^{\binom{(\textbf{k})_i}{2}}w_i^{(\textbf{k})_i},
\end{align*}
for any $i\geq2$ and any p-vector $\Mat{k}$ such that $|\Mat{k}|=i$.
\end{lemma}

\begin{proof}
Let us prove Lemma \ref{l:induction} by induction on the iteration number.

First, consider $i=2$.
When entering the loop for the first time, we have $T_{1}[\vecdelta_j] = w_j$ for each cell $C_j$.
Then, for a particular cell $C_j$ selected on Line 5, there are two cases to consider.

The first case is $\Mat{k}_{old}=\vecdelta_j$.
Then $W_{old} = w_j$ and 
\begin{align*}
    W_{new} &= w_j w_j \left(\prod_{i\in[p]:i\neq j}r_{ji}^{0}\right) r_{jj} = 1 \cdot r_{jj} w_j^2.
\end{align*}
Moreover, $\Mat{k}_{new}=2\vecdelta_j$.
Since this is the only scenario where we obtain such $\Mat{k}_{new}$ and since $\binom{2}{2\vecdelta_j}=1$, we have 
$$T_2[2\vecdelta_j] = \binom{2}{2\vecdelta_j} r_{jj} w_j^2.$$
    
The second possibility is that $\Mat{k}_{old}=\vecdelta_{j'}$, where $j' \neq j$.
Then $W_{old} = w_{j'}$ and $W_{new} = w_j' w_j r_{jj'}$.
The new p-vector $\Mat{k}_{new} = \vecdelta_j + \vecdelta_{j'}$ will also be obtained when the selected cell is $C_{j'}$ and $\Mat{k}_{old}=\vecdelta_{j}$.
The resulting $W_{new}$ will be the same as above.
Those values will be summed together (Line 9) and produce
$$T_2[\vecdelta_j+\vecdelta_{j'}] =
2 \cdot r_{jj'} w_j w_{j'} = \binom{2}{\vecdelta_j+\vecdelta_{j'}}  r_{jj'} w_j w_{j'}.$$
Hence, the lemma holds at the end of the first iteration.

Second, assume the claim holds at the end of iteration $i$.
Let us investigate the entry $T_{i+1}[\Mat{k}]$.
For now, consider $\Mat{k}$ without any zero entries.
Then there are $p$ cases that will produce a particular p-vector $\Mat{k} = (k_1, k_2, \ldots, k_p)$, namely
\begin{align*}
    \Mat{k}_{old} &= (k_1 - 1, k_2, \ldots, k_p) \text{ and cell } C_1 \\
    \Mat{k}_{old} &= (k_1, k_2 -1, \ldots, k_p) \text{ and cell } C_2 \\
    &\vdots\\
    \Mat{k}_{old} &= (k_1, k_2, \ldots, k_p - 1) \text{ and cell } C_p.
\end{align*}
For a particular cell $C_j$ and $\Mat{k}_{old} = \Mat{k} - \vecdelta_j$, we have by induction hypothesis:
\begin{align*}
  W_{old} = \binom{i}{\Mat{k}-\vecdelta_j} r_{jj}^{\binom{(\Mat{k})_j-1}{2}}w_j^{(\Mat{k})_j-1}
  \prod_{i,l\in[p]:i<l,i\neq j \neq l}r_{il}^{(\Mat{k})_i(\Mat{k})_l}
  \prod_{i\in[p]:i \neq j}r_{ii}^{\binom{(\Mat{k})_i}{2}}w_i^{(\Mat{k})_i}
  \prod_{i\in[p]:i\neq j}r_{ji}^{((\Mat{k})_j-1)(\Mat{k})_i}.
\end{align*}
Following the weight update on Line 7, this value will become
\begin{align*}
  W_{new} &=\binom{i}{\Mat{k}-\vecdelta_j} r_{jj}^{\binom{(\Mat{k})_j-1}{2}+(\Mat{k})_j}w_j^{(\Mat{k})_j-1+1}\\
  &\prod_{i,l\in[p]:i<l,i\neq j \neq l}r_{il}^{(\Mat{k})_i(\Mat{k})_l}
  \prod_{i\in[p]:i \neq j}r_{ii}^{\binom{(\Mat{k})_i}{2}}w_i^{(\Mat{k})_i}
  \prod_{i\in[p]:i\neq j}r_{ji}^{((\Mat{k})_j-1)(\Mat{k})_i+(\Mat{k})_i}.
\end{align*}
Manipulating the powers and using the property $r_{ij}=r_{ji}$, we obtain
\begin{align*}
  W_{new} &=\binom{i}{\Mat{k}-\vecdelta_j} \prod_{i,l\in[p]:i<l}r_{il}^{(\Mat{k})_i(\Mat{k})_l}
  \prod_{i\in[p]}r_{ii}^{\binom{(\Mat{k})_i}{2}}w_i^{(\Mat{k})_i}.
\end{align*}
Observe that the product after the multinomial coefficient will be the same for any of the $p$ cases outlined above.
Hence, the final new table entry is given by
\begin{align*}
    T_{i+1}[\Mat{k}] &= \sum_{j=1}^p\binom{i}{\Mat{k}-\vecdelta_j}
    \prod_{i,l\in[p]:i<l}r_{il}^{(\Mat{k})_i(\Mat{k})_l}
  \prod_{i\in[p]}r_{ii}^{\binom{(\Mat{k})_i}{2}}w_i^{(\Mat{k})_i}\\
  &= \prod_{i,l\in[p]:i<l}r_{il}^{(\Mat{k})_i(\Mat{k})_l}
  \prod_{i\in[p]}r_{ii}^{\binom{(\Mat{k})_i}{2}}w_i^{(\Mat{k})_i}
  \sum_{j=1}^p\binom{i}{\Mat{k}-\vecdelta_j}\\
  &=\binom{i+1}{\Mat{k}}
  \prod_{i,l\in[p]:i<l}r_{il}^{(\Mat{k})_i(\Mat{k})_l}
  \prod_{i\in[p]}r_{ii}^{\binom{(\Mat{k})_i}{2}}w_i^{(\Mat{k})_i},
\end{align*}
which is consistent with the claim.

The last thing to consider is if there are some zero entries in $\Mat{k}$.
Suppose there are $z$ of them and w.l.o.g. assume they are on the positions $(p-z+1), (p-z+2), \ldots p$.
Then we obtain a result such that
\begin{align*}
T_{i+1}[\Mat{k}] =
\prod_{i,j\in[p]:i<j} r_{ij}^{(\textbf{k})_i(\textbf{k})_j} \prod_{i=1}^{p}r_{ii}^{\binom{(\textbf{k})_i}{2}}w_i^{(\textbf{k})_i} \sum_{j=1}^{p-z}\binom{i}{\Mat{k}-\vecdelta_j}.
\end{align*}
Denote $\Mat{u}$ the first $p-z$ components of the vector $\Mat{k}-\vecdelta_j$.
Note that $\binom{i}{\Mat{k}-\vecdelta_j}=\binom{i}{\Mat{u}}$, since the last $z$ entries are all zeros.
Hence, even now it holds that
\begin{align*}
T_{i+1}[\Mat{k}] &= \binom{i+1}{\Mat{k}}\prod_{i,j\in[p]:i<j} r_{ij}^{(\textbf{k})_i(\textbf{k})_j} \prod_{i=1}^{p}r_{ii}^{\binom{(\textbf{k})_i}{2}}w_i^{(\textbf{k})_i}.
\end{align*}
\end{proof}

\begin{theorem}
\label{th:algo-fo2}
Algorithm \ref{algo:dp} computes $\wfomc{}(\phi, n, w, \overline{w})$ of a universally quantified \fotwo{} sentence $\phi$ in prenex normal form.
Moreover, it does so in time polynomial in the domain size $n$.
\end{theorem}

\begin{proof}
By Lemma \ref{l:induction}, we have
\begin{align*}
    T_n[\Mat{k}] = \binom{n}{\textbf{k}} \prod_{i,j\in[p]:i<j} r_{ij}^{(\textbf{k})_i(\textbf{k})_j} \prod_{i=1}^{p}r_{ii}^{\binom{(\textbf{k})_i}{2}}w_i^{(\textbf{k})_i}.
\end{align*}
On Line 13, all those entries are summed together which produces a formula identical to the one in Equation \ref{eq:2wfomc}.

As for the second part of the claim.
The first loop on lines $1-3$ runs in time \BigO{1} with respect to $n$.
The large loop on lines $4-12$ runs in \BigO{n}.
The first nested loop (lines $5-11$) is again independent of $n$, and the second (lines $6-10$) runs in \BigO{n^p}.
The final sum on Line 13 also runs in \BigO{n^p}.
Overall, we have
\begin{align*}
    \BigO{n} \cdot \BigO{n^p} + \BigO{n^p} \in \BigO{n^{p+1}},
\end{align*}
which is polynomial in the domain size $n$.
\end{proof}

\subsection{Enforcing a Linear Order}
When adding the linear order axiom to the input sentence $\psi$, each model of $\psi$ will be with respect to some domain ordering.
Assume we find the set $\Omega$ of all models for one fixed ordering.
Having a domain permutation $\pi$, $$\Omega' = \bigcup_{\omega\in\Omega}\Set{\pi(\omega)}$$ will be the set of all models with respect to the new domain ordering defined by $\pi$.
Hence, the situation is symmetric for any particular ordering of the domain.

\begin{theorem}
\label{th:n-factorial}
Let $\phi$ be a formula of the form $\phi = \psi(x,y) \wedge Linear(\leq)$, where $\psi(x,y)$ is a universally quantified \fotwo{} sentence and $\leq$ is one of its predicates.
Let $\Delta$ be a domain over which we want to compute \wfomc{}.

If $\omega \models \phi$ and $\pi$ is a permutation of $\Delta$, such that $\pi(\Delta) \neq \Delta$, then $\pi(\omega) \models \phi$, where application of $\pi$ to a possible world is defined by appropriate substitution of the domain elements in ground atoms. Moreover, $\omega \neq \pi(\omega)$.
\end{theorem}

\begin{proof}
If $\omega$ is a model of $\phi$, we can partition $\omega$ into two disjoint sets:
$\omega[\leq]$ holding only atoms with the predicate $\leq$ and $\omega_\psi = \omega  \setminus \omega[\leq]$.
$\omega[\leq]$ defines an ordering of $\Delta$ and $\omega_\psi$ is then a model of $\forall x \forall y: \psi(x,y)$ respecting the ordering defined by $\omega[\leq]$.
Applying the permutation $\pi$ to $\omega[\leq]$ will define a different domain ordering.

Since there are no constants in $\phi$, $\pi(\omega_\psi)$ will still be a model of $\forall x \forall y: \psi(x,y)$ (we simply apply a different substitution to the variables in $\psi$).
Moreover, since $\omega_\psi$ respected the ordering defined by $\omega[\leq]$,$\pi(\omega_\psi)$ will respect the new ordering defined by $\pi(\omega[\leq])$.

Hence $\pi(\omega) = \pi(\omega[\leq]) \cup \pi(\omega_\psi)$ is another model of $\phi$ and it must be different from $\omega$, because $\pi(\omega[\leq])$ defines a different ordering than $\omega[\leq]$.
\end{proof}

\begin{corollary}
\label{cor:n-factorial}
To compute $\wfomc{}(\phi, n, w, \overline{w})$, where $\phi = \psi(x,y) \wedge Linear(\leq)$, we can compute \wfomc{} for one ordered domain of size $n$ and then multiply the result by the factorial of $n$, since there are $n!$ different permutations of the domain.
\end{corollary}

\noindent Let us now show that we can compute \wfomc{} of a formula $\phi=\psi\wedge Linear(\leq)$ for a fixed domain ordering using only slightly modified Algorithm \ref{algo:dp}.
The modified algorithm will take advantage of the fact that when we are processing the $i$-th domain element, it holds that $i' < i$ for all already processed domain elements $i'$.
Hence, when extending the domain by the constant $i$ (and consequently, extending the models by atoms containing $i$), the only difference will be in the models of the subformulas $\psi_{ij}(A, B)$, where $A,B\in\Delta$.
The one constant must be ``greater'' than the other in the sense of the enforced domain ordering.
Thus, we only need to redefine $r_{ij}$ to reflect this.
Then, we may prove that \fotwo{} with a linear order axiom is domain-liftable in a similar manner to how we proved correctness of Algorithm \ref{algo:dp} for \fotwo{} alone.

Let us redefine $r_{ij}$ as
\begin{align}
\label{def:new-r}
  r_{ij}= \wmc(\psi_{ij}(A,B)\wedge(B \leq A)\wedge \neg(A \leq B), w', \overline{w}')  
\end{align}

\begin{theorem}
\label{th:algo-fo2+lo}
\iwfomc{} with $r_{ij}$ values from Equation \ref{def:new-r} computes $\wfomc(\phi, n, w, \overline{w})$ of a universally quantified \fotwo{} sentence $\phi$ in prenex normal form on the ordered domain $\Delta = \Set{1\leq 2 \leq \ldots \leq n}$.
Moreover, it does so in time polynomial in the domain size $n$.
\end{theorem}

\begin{proof}
Let us prove the claim by induction on size of the domain.

The base step is analogical to the one in proof of Lemma \ref{l:induction}.
More generally speaking, for a domain of a constant size $K$ ($K=1$ in Algorithm \ref{algo:dp}), we may simply ground the problem and compute its \wmc{} without any lifting.
Since $K$ is a constant with respect to $n$, we won't exceed the polynomial running time.

The inductive step differs from the one for Lemma \ref{l:induction}, but still builds on the same intuition.
Now, assume that our algorithm computes $\wfomc$ with linear order for a domain of size $i$, where the result is stored as the table entries $T_i[\Mat{k}]$ for all p-vectors $\Mat{k}$ such that $|\Mat{k}|=i$ (the final result would be obtained by summing those entries together).
Consider processing of the element $(i+1)$.
For a particular cell $C_j$ and a p-vector $\Mat{k}$, adding the new element will again extend the existing models with new atoms.
First, atoms corresponding to the subformula $\psi_j(i+1)$ will be added, hence the old weight must be multiplied by $w_j$.
Second, atoms corresponding to the subformulas $\psi_{jk}(i+1,i')$ for each cell $C_k$ and each processed element $i' (1 \leq i' < i)$.
However, only possible worlds satisfying $i' < i+1$ on top of that, will be models of the input sentence with respect to the fixed domain ordering.
That is precisely captured by $r_{ij}$ from Equation \ref{def:new-r}.
Other possible worlds will be assigned zero weight.
Hence, $$W_{new} = W_{old} \cdot w_j \cdot \prod_{l=1}^pr_{jl}^{(\Mat{k})_l}.$$
There are more possible p-vectors $\Mat{u}$ and cells $C_m$ such that $\Mat{u} + \vecdelta_m = \Mat{k} + \vecdelta_j = \Mat{k}_{new}$.
Those all correspond to different, mutually independent models whose weights can be added together.
Since we are processing all possible p-vectors, those also correspond to the only existing models.

Therefore, at the end of the final iteration, we will have summed up weights of all existing models of size $n$.
And since we only substituted one value in the original Algorithm \ref{algo:dp}, the computation still runs in time polynomial in the domain size.
\end{proof}

\begin{theorem}
The language of \fotwo{} extended by a linear order axiom is domain-liftable.
\end{theorem}

\begin{proof}
For an input sentence $\phi = \psi\wedge Linear(\leq)$, where $\psi$ is an \fotwo{} sentence, start with converting $\psi$ to a prenex normal form with each predicate having arity at most $2$ \citep{Logic/FO2-decidable}.
Then apply the skolemization for \wfomc{} \citep{Lifting/Skolemization} to obtain a sentence of the form $\phi = \forall x\forall y:\psi(x,y)\wedge Linear(\leq)$, where $\psi$ is a quantifier-free formula.

By Theorem \ref{th:algo-fo2+lo}, we know that Algorithm \ref{algo:dp} computes $\wfomc(\phi, n, w, \overline{w})$ for one fixed ordering of the domain in time polynomial with respect to the domain size.
Once we have that value, we may multiply it by $n!$ to obtain the overall \wfomc, as is stated in Corollary \ref{cor:n-factorial}.
The entire computation thus runs in time polynomial in the domain size.
\end{proof}

\subsubsection*{A Worked Example of IncrementalWFOMC}
Let us now use another example of splitting a sequence to demonstrate the work of Algorithm \ref{algo:dp}.
Consider the sentence $\phi = \forall x \forall y: \psi(x, y) \wedge Linear(\leq)$, 
where $\psi$ is the conjunction of
\begin{align*}
    &\neg H(x) \vee \neg T(x),\\
    &H(y) \wedge (x\leq y) \Rightarrow H(x),\\
    &T(x) \wedge (x\leq y) \Rightarrow T(y).
\end{align*}

This time, we model a three-way split of a sequence, differentiating its head, tail and middle.
We have already seen the third formula, which defines a property of the sequence tail.
The second formula does the same for the head.
We also require that for each element, at least one of $H/1, T/1$ is set to false.
If both were set to true, then one element should be part of both the head and the tail, which is obviously something, we do not want.
If they are both set to false, then the element is part of the sequence middle.

Our goal is to compute $\wfomc(\phi, n, w, \overline{w})$, where $(w, \overline{w})$ are some weight functions.
For more clarity in the computations below, we leave the weights as parameters (except for the $\leq$ predicate, whose weights are fixed to one).
We will substitute concrete numbers at the end of our example.

First, we construct valid cells of $\psi$.
There are $3$ in total:
\begin{align*}
    C_1(x) &= H(x) \wedge \neg T(x) \wedge (x\leq x) \\
    C_2(x) &= \neg H(x) \wedge T(x) \wedge (x\leq x) \\
    C_3(x) &= \neg H(x) \wedge \neg T(x) \wedge (x\leq x)
\end{align*}
Having valid cells, we need to compute the values $r_{ij}$ and $w_k$.
Since we left the input weight functions as parameters, those cannot be specified numerically.
Instead, we use the following symbols:
\begin{equation*}
\begin{aligned}[c]
w = \begin{pmatrix}
w_1 \\
w_2 \\
w_3
\end{pmatrix}
\end{aligned}
\qquad
\begin{aligned}[c]
R = \begin{pmatrix}
r_{11} & r_{12} & r_{13} \\
r_{21} & r_{22} & r_{23} \\
r_{31} & r_{32} & r_{33}
\end{pmatrix}
\end{aligned}
\end{equation*}

Finally, we can start with the pseudocode.
Following the loop on Lines 1--3, we obtain the table $T_1$ as follows:
\begin{align*}
    T_1[(1, 0, 0)] &= w_1\\
    T_1[(0, 1, 0)] &= w_2\\
    T_1[(0, 0, 1)] &= w_3
\end{align*}

For the main loop on Lines 4--12, we have $i = [2, 3]$ and $j = [1, 2, 3]$.
\begin{itemize}
    \item Set $i = 2$.
    
    \begin{itemize}
        \item Set $j = 1$. Now we iterate over entries in $T_1$.
    
    First, we have $\Mat{k}_{old} = (1,0,0)$ and $W_{old}=w_1$.

We compute the new weight as $$W_{new} \leftarrow W_{old} \cdot w_1 \cdot r_{11}^1 \cdot r_{12}^0 \cdot r_{13}^0 = w_1^2r_{11}.$$

The new p-vector will be $\Mat{k}_{new} \leftarrow (2, 0, 0)$.

The old value $T_2[(2, 0, 0)] = 0$.

Hence, we will set
$$T_2[(2, 0, 0)] \leftarrow 0 + w_1^2r_{11}.$$

Second, we have $\Mat{k}_{old} = (0,1,0)$ and $W_{old}=w_2$.

That will lead to
$$T_2[(1, 1, 0)] \leftarrow 0 + w_1w_2r_{12}.$$

Third, $\Mat{k}_{old} = (0,0,1)$ and $W_{old}=w_3$.
Now, we perform an update
$$T_2[(1, 0, 1)] \leftarrow 0 + w_1w_3r_{13}.$$

    \item Set $j=2$. Again, iterate over entries in $T_1$.
    
    First, we have $\Mat{k}_{old} = (1,0,0)$ and $W_{old}=w_1$.
    
    We compute the new weight as $$W_{new} \leftarrow W_{old} \cdot w_2 \cdot r_{21}^1 \cdot r_{22}^0 \cdot r_{23}^0 = w_1w_2r_{21}.$$
    
    The new p-vector $\Mat{k}_{new} \leftarrow (1, 1, 0)$ already has non-zero value set in $T_2$, i.e., $$T_2[(1, 1, 0)] = w_1w_2r_{12}.$$
    
    Hence, we will now assign
    $$T_2[(1, 1, 0)] \leftarrow w_1w_2r_{12} + w_1w_2r_{21},$$
    which we will factor into
     $$T_2[(1, 1, 0)] = w_1w_2(r_{12}+r_{21}).$$
     
     We proceed analogically for $\Mat{k}_{old} = (0,1,0), W_{old}=w_2,$, leading to
     $$T_2[(0, 2, 0)] \leftarrow 0 + w_2^2r_{22},$$
     and for $\Mat{k}_{old} = (0,0,1), W_{old}=w_3.$, leading to
     $$T_2[(0, 1, 1)] \leftarrow 0 + w_2w_3r_{23}.$$
     
     \item After repeating the steps for $j=3$, we arrive at the complete table $T_2$ with entries:
     \begin{align*}
        T_2[(2, 0, 0)] &= w_1^2 r_{11}\\
        T_2[(1, 1, 0)] &= w_1 w_2 (r_{12} + r_{21})\\
        T_2[(1, 0, 1)] &= w_1 w_3 (r_{13} + r_{31})\\
        T_2[(0, 2, 0)] &= w_2^2 r_{22}\\
        T_2[(0, 1, 1)] &= w_2 w_3 (r_{23} + r_{32})\\
        T_2[(0, 0, 2)] &= w_3^2 r_{33}
     \end{align*}
    \end{itemize}
    
    \item When performing the computation for $i=3$, we now iterate over entries in $T_2$.
    Hence, for each $j$, there will now be six p-vector keys and their respective values to process.
    
    Eventually, we arrive at $T_3$ such that
    \begin{align*}
        T_3[(3, 0, 0)] &= w_1^3 r_{11}^3\\
        T_3[(2, 1, 0)] &= w_1^2 w_2 r_{11} [r_{12} (r_{12} + r_{21}) + r_{21}^2]\\
        T_3[(2, 0, 1)] &= w_1^2 w_3 r_{11} [r_{13} (r_{13} + r_{31}) + r_{31}^2]\\
        T_3[(1, 2, 0)] &= w_1 w_2^2 r_{22} [r_{21} (r_{21} + r_{12}) + r_{12}^2]\\
        T_3[(1, 1, 1)] &= w_1 w_2 w_3 [r_{12} r_{13} (r_{23} +r_{32}) + r_{21} r_{23} (r_{13} + r_{31}) + r_{31} r_{32} (r_{12} + r_{21})]\\
        T_3[(1, 0, 2)] &= w_1 w_3^2 r_{33} [r_{31} (r_{31} + r_{13}) + r_{13}^2]\\
        T_3[(0, 3, 0)] &=w_2^3 r_{22}^3\\
        T_3[(0, 2, 1)] &= w_2^2 w_3 r_{22} [r_{23} (r_{23} + r_{32}) + r_{32}^2]\\
        T_3[(0, 1, 2)] &= w_2 w_3^2 r_{33} [r_{32} (r_{32} + r_{23}) + r_{23}^2]\\
        T_3[(0, 0, 3)] &=w_3^3 r_{33}^3\\
     \end{align*}
\end{itemize}

Per Line 13, the final result is obtained by summing all the values in $T_3$ that are written above.

As is stated in Theorem \ref{th:algo-fo2+lo}, the obtained value is \wfomc{} for one particular ordering of the domain (specifically, the ordering $1 \leq 2 \leq 3$).
Since the result will be the same for any ordering of the domain, multiplying the value by $n!=6$ will produce the final \wfomc{} value.

Let us now check the obtained result by comparing it to a purely combinatorial solution of the problem.
To simplify matters a little, we assume to be working only with the particular ordering $1 \leq 2 \leq 3$, which allows us to disregard the multiplying by $n!$.

To find the number of three-way sequence splits, we set all weights to one.
For unitary weights, we obtain
\begin{equation*}
\begin{aligned}[c]
\begin{pmatrix}
w_1 \\
w_2 \\
w_3
\end{pmatrix}
=
\begin{pmatrix}
1 \\
1 \\
1
\end{pmatrix},
\end{aligned}
\qquad
\begin{aligned}[c]
\begin{pmatrix}
r_{11} & r_{12} & r_{13} \\
r_{21} & r_{22} & r_{23} \\
r_{31} & r_{32} & r_{33}
\end{pmatrix}
=
\begin{pmatrix}
1 & 0 & 0 \\
1 & 1 & 1 \\
1 & 0 & 1 \\
\end{pmatrix}.
\end{aligned}
\end{equation*}
Plugging those values into $T_3$ and summing produces
$$\sum_{\Mat{k}\in\Nat^3:|\Mat{k}|=3}T_3[\Mat{k}] =10.$$

The combinatorial solution may be found, e.g., by using the popular \textit{stars and bars} method:
\begin{center}
    \begin{tikzpicture}[scale=.7]
  \def\total{3}
  \xdef\y{0}
  \foreach \i in {0,...,\total}{
    \pgfmathsetmacro{\maxj}{int(\total-\i)}
    \foreach \j in {0,...,\maxj}{
      \pgfmathparse{int(\y+1)}\xdef\y{\pgfmathresult}
      \foreach \k in {1,...,\total}
        \node at (\k,\y){$\star$};
      \node at (.45+\j+\i,\y) {$\big|$};
      \node at (.55+\i,\y) {$\big|$};
    }
  }
\end{tikzpicture}
\end{center}
As we can see, there are indeed $10$ ways to split a particular sequence in this way.

\subsection{Domain-Liftability of \ctwo{} with Linear Order}
\wfomc{} in \ctwo{} may be reduced to \wfomc{} in \fotwo{} under cardinality constraints.
\wfomc{} under cardinality constraints may then be solved by repeated calls to a \wfomc{} oracle.
As there will only be a polynomial number of such calls in the domain size, it follows that \fotwo{} with cardinality constraints and also \ctwo{} are domain-liftable
\citep{Lifting/C2-domain-liftable}.

Since the \ctwo{} domain-liftability proof only relies on a domain-lifted \wfomc{} oracle, we may use our new algorithm for computing \wfomc{} with linear order as that oracle, leading to our final result.

\begin{theorem}
\label{th:c2+lo}
The language of \ctwo{} extended by a linear order axiom is domain-liftable.
\end{theorem}

\noindent We omit the proof as it would consist of almost word by word restating of the already available proof on domain-liftability of \ctwo{} \citep{Lifting/C2-domain-liftable} with only cosmetic changes.

\subsection{Predecessor Relations}
Having enforced a domain ordering using the linear order axiom, we may define more complicated relations.
Once we have ordered the domain, a natural question to ask for a constant $A\in \Delta$ is: \textit{What element is the (immediate) predecessor of $A$}?
That question can even be further generalized to: \textit{What element is the k-th predecessor of $A$}?
In the subsequent paragraphs, we present possible encodings of the predecessor and the predecessor of predecessor relations for WFOMC.

\subsubsection*{Predecessor Relation}
Denote $Pred(x, y)$ the relation that $x$ is the (immediate) predecessor of $y$ with respect to a linear ordering of the domain enforced by the predicate $\leq$.
To properly encode $Pred/2$, we make use of an auxiliary relation $Perm/2$, which defines a specific permutation of the domain elements.

We claim that the predecessor relation can be encoded using the following theory:
\begin{align}
    \Psi_{Pred} = \{ &\forall x: \neg Perm(x, x),\label{eq:pred:1}\\
    &\forall x \exists^{=1}y: Perm(x, y),\label{eq:pred:2}\\
    &\forall y \exists^{=1}x: Perm(x, y),\label{eq:pred:3}\\
    &\forall x \forall y: Pred(x, y) \Rightarrow Perm(x, y),\label{eq:pred:4}\\
    &\forall x \forall y: Pred(x, y) \Rightarrow (x \leq y),\label{eq:pred:5}\\
    &|Pred| = n - 1\label{eq:pred:6} \}
\end{align}

Let us investigate the correctness of the encoding.
Consider the domain elements to be nodes of a graph and a domain ordering to be the topological ordering.
Relations will then add edges to the graph.
We provide visualisations for a $5$-element domain.

We start without any relations:
\begin{center}
\begin{tikzpicture}
	\node (1) at (-2,0) [shape=circle,draw] {1};
	\node (2) at (-1,0) [shape=circle,draw] {2};
 	\node (3) at ( 0,0) [shape=circle,draw] {3};
	\node (4) at ( 1,0) [shape=circle,draw] {4};
	\node (5) at ( 2,0) [shape=circle,draw] {5};
\end{tikzpicture}
\end{center}

It is obvious that we would like to achieve the situation when our graph looks like
\begin{center}
\begin{tikzpicture}
	\node (1) at (-2,0) [shape=circle,draw] {1};
	\node (2) at (-1,0) [shape=circle,draw] {2}
		edge [<-] (1);
 	\node (3) at ( 0,0) [shape=circle,draw] {3}
		edge [<-] (2);
	\node (4) at ( 1,0) [shape=circle,draw] {4}
		edge [<-] (3);
	\node (5) at ( 2,0) [shape=circle,draw] {5}
		edge [<-] (4)
		edge [->, bend left=30,dashed] (1);
\end{tikzpicture}
\end{center}
where the edges drawn by a full line correspond to the $Pred$ relation and all of the edges correspond to the $Perm$ relation.
Let us now investigate the need for each of the formulas to guarantee such graph structure.

Sentence \ref{eq:pred:1} prohibits loops for $Perm$ and sentences \ref{eq:pred:2} and \ref{eq:pred:3} require $Perm$ to be a bijection.
Hence, $Perm$ must be a permutation without fixed points of the domain.\footnote{Permutations without fixed points are also known as {\em derangements}.}
Nevertheless, more is needed since various (undesired) structures satisfy that requirement.
For instance:
\begin{center}
\begin{tikzpicture}
	\node (1) at (-2,0) [shape=circle,draw] {1};
	\node (2) at (-1,0) [shape=circle,draw] {2}
		edge [<-, bend right=40] (1)
		edge [->, bend left=40] (1);
 	\node (3) at ( 0,0) [shape=circle,draw] {3}
		edge [<-, bend right=40] (5);
	\node (4) at ( 1,0) [shape=circle,draw] {4}
		edge [<-, bend right=40] (3);
	\node (5) at ( 2,0) [shape=circle,draw] {5}
		edge [<-, bend right=40] (4);
\end{tikzpicture}
\end{center}

Sentences \ref{eq:pred:5}, \ref{eq:pred:6} require that the \textit{edges} of $Pred$ never go \textit{right to left} and that there are exactly $n-1$ of them.
A graph such as
\begin{center}
\begin{tikzpicture}
	\node (1) at (-2,0) [shape=circle,draw] {1};
	\node (2) at (-1,0) [shape=circle,draw] {2};
 	\node (3) at ( 0,0) [shape=circle,draw] {3}
		edge [<-, bend right=45] (1)
		edge [<-] (2)
		edge [->, loop above] (3);
	\node (4) at ( 1,0) [shape=circle,draw] {4};
	\node (5) at ( 2,0) [shape=circle,draw] {5}
		edge [<-] (4);
\end{tikzpicture}
\end{center}
satisfies such constraints.

Finally, sentence \ref{eq:pred:4} introduces a relationship between $Pred$ and $Perm$.
Whenever $Pred(A, B)$ is satisfied, so must be $Perm(A, B)$.
As an immediate consequent, there must be $n-1$ edges of $Perm$ that go \textit{left to right} (they must go right since loops are prohibited).
Moreover, $Perm$ must be a bijection so all of the $n-1$ edges must have different starting node and end node.
There is only one way, how to connect the nodes now:
\begin{center}
\begin{tikzpicture}
	\node (1) at (-2,0) [shape=circle,draw] {1};
	\node (2) at (-1,0) [shape=circle,draw] {2}
		edge [<-] (1);
 	\node (3) at ( 0,0) [shape=circle,draw] {3}
		edge [<-] (2);
	\node (4) at ( 1,0) [shape=circle,draw] {4}
		edge [<-] (3);
	\node (5) at ( 2,0) [shape=circle,draw] {5}
		edge [<-] (4);
\end{tikzpicture}
\end{center}

The relation $Perm$ still requires one more edge to be added.
It must be $Perm(n,1)$, since $n$ is the only element for which, in terms of (bijective) functions, we still do not have an image defined, and $1$ is the only element which is not yet an image of any other element.
Thus, we arrive at our desired graph:
\begin{center}
\begin{tikzpicture}
	\node (1) at (-2,0) [shape=circle,draw] {1};
	\node (2) at (-1,0) [shape=circle,draw] {2}
		edge [<-] (1);
 	\node (3) at ( 0,0) [shape=circle,draw] {3}
		edge [<-] (2);
	\node (4) at ( 1,0) [shape=circle,draw] {4}
		edge [<-] (3);
	\node (5) at ( 2,0) [shape=circle,draw] {5}
		edge [<-] (4)
		edge [->, bend left=30, dashed] (1);
\end{tikzpicture}
\end{center}

\begin{lemma}
\label{lm:pred}
The first-order theory $\Psi_{Pred}$ along with a linear order enforcing predicate $\leq$ correctly defines the immediate predecessor relation for any domain size $n\geq 2$.
Moreover, the theory has exactly $n!$ models.
\end{lemma}

\begin{proof}
By the reasoning above, for any domain size $n \geq 2$ and the domain ordering $1 \leq 2\leq\ldots\leq n$, the theory $\Psi_{Pred}$ has exactly one model $\omega_{Pred}$ such that
\begin{align*}
    \omega_{Pred} = \bigcup_{i=1}^{n-1} &\Set{Perm(i, i+1), Pred(i, i+1)} \cup \Set{Perm(n,1)}.
\end{align*}
Hence, for every element $i\geq 2$, its predecessor is the element $i-1$.
The element $i=1$ has no predecessor.
That is the immediate predecessor relation.
It follows from Theorem \ref{th:n-factorial} that $\Psi_{Pred}$ defines the predecessor correctly for any domain ordering.

For any domain ordering, $\Psi_{Pred}$ has exactly one model.
Since there are $n!$ possible orderings, there are $n!$ models.
\end{proof}

\noindent
Since we are able to define the predecessor relation, we may extend the linear order axiom to also capture the predecessor property.

\begin{definition}
Let $\psi$ be a logical sentence possibly containing binary predicates $\leq$ and $Pred$.
A possible world $\omega$ is a model of $\phi = \psi \wedge Linear(\leq, Pred)$ if and only if $\omega$ is a model of $\psi \wedge Linear(\leq)$,
and the relation $\omega[{Pred}]$ forms the immediate predecessor relation w.r.t.\ the order $\leq$.%
\end{definition}

\begin{theorem}
$\wfomc{}(\psi \wedge Linear(\leq, Pred),n,w,\barw)$, where $\psi$ is an arbitrary \ctwo{} sentence, can be computed in time polynomial in $n$.
\end{theorem}
\begin{proof}
By Lemma \ref{lm:pred}, we can express the predecessor relation using the theory $\Psi_{Pred}$.
Hence, the computation is equivalent to computing $\wfomc{}(\psi \wedge\Psi_{Pred} \wedge Linear(\leq),n,w,\barw)$.
Since $\psi \wedge\Psi_{Pred}$ is a \ctwo{} sentence, by Theorem \ref{th:c2+lo}, we are computing \wfomc{} over a domain-liftable language.
\end{proof}

\noindent
Before we can use $\Psi_{Pred}$ as a part of our algorithm's input, we need to further encode it using the language  of universally quantified \fotwo{} with cardinality constraints.
The counting quantifiers from sentences \ref{eq:pred:2} and \ref{eq:pred:3} may be reduced to ordinary existential quantifiers by adding a single cardinality constraint $(|Perm|=n)$ \citep{Lifting/C2-domain-liftable}.
Afterwards, the sentences need to be \textit{skolemized} \citep{Lifting/Skolemization}.
Overall, we end up with the theory
\begin{align*}
    \Psi_{Pred}' = \{ &\forall x: \neg Perm(x, x),\\
    &\forall x \forall y: \neg Perm(x, y) \vee S_1(x),\\
    &\forall x \forall y: \neg Perm(x, y) \vee S_2(x),\\
    &\forall x \forall y: Pred(x, y) \Rightarrow Perm(x, y),\\
    &\forall x \forall y: Pred(x, y) \Rightarrow (x \leq y),\\
    &|Perm| = n,\\
    &|Pred| = n - 1 \},
\end{align*}
where $S_1/1$ and $S_2/1$ are fresh (Skolem) predicates such that $w(S_1) = w(S_2) = 1$ and $\overline{w}(S_1) = \overline{w}(S_2) = -1$.

\subsubsection*{Predecessor of the Predecessor}
Once we have found the predecessor, we may seek predecessor of that predecessor.
Let us denote such relation $Pred2(x, y)$, i.e., $Pred2(x, y)$ is true if and only if there exists an element $z$ such that $Pred(x, z)$ and $Pred(z, y)$ are true (with respect to a linear order enforcing predicate $\leq$).

Following a similar reasoning as for definition of $Pred(x,y)$, we will start by defining a permutation of the domain elements.
Now, the permutation will consist of two cycles, one of length $\lfloor\frac{n}{2}\rfloor$ and the other of length $\lceil\frac{n}{2}\rceil$.

Denote the new relation $Perm2(x, y)$.
Let us start by saying that $Perm2$ should be a permutation without fix-points:
\begin{align}
    &\forall x: \neg Perm2(x, x)\label{eq:pred2:1}\\
    &\forall x \exists^{=1} y: Perm2(x, y)\label{eq:pred2:2}\\
    &\forall y \exists^{=1} x: Perm2(x, y)\label{eq:pred2:3}
\end{align}
Next, we need to track how many times we go \textit{right to left}.
There should be exactly two transitions like that.
Let us enforce that by
\begin{align}
    &\forall x \forall y: Inv(x, y) \Leftrightarrow ((y \leq x) \wedge Perm2(x, y)),\label{eq:pred2:4}\\
    &|Inv| = 2\label{eq:pred2:5}.
\end{align}
Obviously, that prohibits more than two cycles but there could still be just one, such as
\begin{center}
\begin{tikzpicture}
	\node (1) at (-2,0) [shape=circle,draw] {1};
	\node (2) at (-1,0) [shape=circle,draw] {2};
 	\node (3) at ( 0,0) [shape=circle,draw] {3}
		edge [<-, bend right=45] (1)
		edge [->, bend left=30] (2);
	\node (4) at ( 1,0) [shape=circle,draw] {4}
		edge [<-, bend right=45] (2);
	\node (5) at ( 2,0) [shape=circle,draw] {5}
	    edge [<-] (4);
	\node (6) at ( 3,0) [shape=circle,draw] {6}
		edge [<-] (5)
		edge [->, bend left=30] (1);
\end{tikzpicture}
\end{center}

We will prevent one cycle by differentiating odd and even nodes.
We can do that by coloring the nodes with two different colors such that neighboring nodes are colored differently (we will require the relation $Pred$ for that):
\begin{align}
    &\forall x : Red(x) \vee Blue(x)\label{eq:pred2:6}\\
    &\forall x : \neg Red(x) \vee \neg Blue(x)\label{eq:pred2:7}\\
    &\forall x \forall y : Red(x) \wedge Pred(x,y) \Rightarrow Blue(y)\label{eq:pred2:8}\\
    &\forall x \forall y : Blue(x) \wedge Pred(x,y) \Rightarrow Red(y)\label{eq:pred2:9}
\end{align}
When counting the models, we just need to keep in mind that there are two ways how to color the sequence (starting with red or with blue). Hence, the (weighted) model count needs to be divided by $2$ in the end.

Having labeled immediate neighbors by different colors, we can enforce that only the same-colored nodes are connected by $Perm2$:
\begin{align}
    &\forall x \forall y : Red(x) \wedge Perm2(x,y) \Rightarrow Red(y)\label{eq:pred2:10}\\
    &\forall x \forall y : Blue(x) \wedge Perm2(x,y) \Rightarrow Blue(y)\label{eq:pred2:11}
\end{align}

Finally, we can relate $Pred2$ and $Perm2$ same as we did in the case of the predecessor relation:
\begin{align}
    &\forall x \forall y : Pred2(x,y) \Rightarrow Perm2(x,y)\label{eq:pred2:12}\\
    &\forall x,y : Pred2(x,y) \Rightarrow (x \leq y)\label{eq:pred2:13}\\
    &|Pred2| = n-2\label{eq:pred2:14}
\end{align}
And we finally arrive at the desired situation:
\begin{center}
\begin{tikzpicture}
	\node (1) at (-2,0) [shape=circle,draw] {1};
	\node (2) at (-1,0) [shape=circle,draw] {2};
 	\node (3) at ( 0,0) [shape=circle,draw] {3}
		edge [<-, bend right=45] (1);
	\node (4) at ( 1,0) [shape=circle,draw] {4}
		edge [<-, bend right=45] (2);
	\node (5) at ( 2,0) [shape=circle,draw] {5}
	    edge [<-, bend right=45] (3)
	    edge [->, bend left=30, dashed] (1);
	\node (6) at ( 3,0) [shape=circle,draw] {6}
		edge [<-, bend right=45] (4)
		edge [->, bend left=30, dashed] (2);
\end{tikzpicture}
\end{center}

Let us concentrate the sentences \ref{eq:pred:1} through \ref{eq:pred2:14} into a first-order theory $\Psi_{Pred2}$.
\begin{lemma}
\label{lm:pred2}
The first-order theory $\Psi_{Pred2}$ along with a linear order enforcing predicate $\leq$ correctly defines the immediate predecessor of the immediate predecessor relation for any domain size $n\geq 4$.
Moreover, there are exactly $2n!$ models of the theory.
\end{lemma}
\begin{proof}
By the reasoning above, for any domain size $n\geq 4$ and the domain ordering $1 \leq 2\leq\ldots\leq n$, the theory has exactly two models, each being
\begin{align*}
    \omega = \bigcup_{i=1}^{n-2} &\Set{Perm2(i, i+2), Pred2(i, i+2)}\\
    &\cup \Set{Perm2(n-1,a), Perm2(n,b)} \cup \omega_{Pred} \cup \omega_{R} \cup \omega_{B},
\end{align*}
where $a=1, b=2$ if $n$ is even and $a=2, b=1$ if $n$ is odd.

The set $\omega_{Pred}$ is the same as specified in the Proof of Lemma \ref{lm:pred}.
Sets $\omega_R$ and $\omega_B$ determine the coloring and these are also the only parts of $\omega$ where the two models of $\Psi_{Pred2}$ differ.

One model contains atoms such that
\begin{align*}
    \omega_R = \{Red(1), Red(3), \ldots, Red(a)\},\\
    \omega_B = \{Blue(2), Blue(4), \ldots, Blue(b)\},
\end{align*}
where $a=n-1,b=n$ if $n$ is even and $a=n, b=n-1$ if $n$ is odd.

Analogously, the other model contains atoms such that
\begin{align*}
    \omega_R = \{Red(2), Red(4), \ldots, Red(a)\},\\
    \omega_B = \{Blue(1), Blue(3), \ldots, Blue(b)\},
\end{align*}
where $a=n,b=n-1$ if $n$ is even and $a=n-1, b=n$ if $n$ is odd.

Hence, every element $i\geq 3$ has the element $i - 2$ as the predecessor of its predecessor.
It follows from Theorem \ref{th:n-factorial} that $\Psi_{Pred2}$ defines the predecessor of the predecessor correctly for any domain ordering.

For any domain ordering, $\Psi_{Pred2}$ has exactly two models. Since there are $n!$ possible orderings, there are $2n!$ models in total.
\end{proof}

\noindent
Now, we can extend the linear order axiom even further in the same manner as above.
\begin{definition}
Let $\psi$ be a logical sentence possibly containing binary predicates $\leq$, $Pred$ and $Pred2$.
A possible world $\omega$ is a model of $\phi = \psi \wedge Linear(\leq, Pred, Pred2)$ if and only if
$\omega$ is a model of $\psi \wedge Linear(\leq, Pred)$, and the relation $\omega[Pred2]$ forms the immediate predecessor of the immediate predecessor relation w.r.t. the order $\leq$.
\end{definition}

\begin{theorem}
$\wfomc(\psi\wedge Linear(\leq, Pred, Pred2), n, w, \barw)$, where $\psi$ is an arbitrary \ctwo{} sentence, can be computed in time polynomial in $n$.
\end{theorem}
\begin{proof}
We may equivalently compute $\wfomc(\psi\wedge\Psi_{Pred2}\wedge Linear(\leq), n, w, \barw)/2$, which follows from Lemma \ref{lm:pred2}.
Hence, we are computing \wfomc{} over the \ctwo{} language extended by the linear order axiom.
By Theorem \ref{th:c2+lo}, that can be done in time polynomial in $n$.
\end{proof}

\noindent
We believe the encoding can be further generalized to the $k$-th predecessor, but we leave that unproven.
Although the encoding is theoretically interesting, since we express a problem seemingly requiring three logical variables using only two, it is of little practical interest.
Our algorithm's complexity is exponential in the number of cells and the definition of $Pred2$ alone has $32$ valid cells (there are $4$ Skolem predicates and the coloring may be swapped).
For that reason, we also omit any experiments on $Pred2$.

\section{Experiments}
To check our results empirically, as well as to assess how our approach scales, we implemented the proposed algorithm in the Julia programming language \citep{Julia/Julia}.
The implementation follows the algorithmic approach presented in the paper, with one notable exception.
Counting quantifiers and cardinality constraints are not handled by repeated calls to a \wfomc{} oracle and subsequent polynomial interpolation \citep{Lifting/C2-domain-liftable}.
Instead, they are processed by introducing a symbolic variable\footnote{Symbolic weights have also been recently used in probabilistic generating circuits \cite{zhang2021probabilistic} in a similar way to ours.} for each cardinality constraint and computing the polynomial (that would be interpolated) explicitly in a single run of the algorithm.
We made use of the \texttt{Nemo.jl} package \citep{Julia/Nemo} for polynomial representation and manipulation.

\subsection{Inference in Markov Logic Networks}
Using \iwfomc{}, we can perform exact lifted probabilistic inference over Markov Logic Networks that use the language of \ctwo{} with the linear order axiom.
We propose one such network over a random graph model similar to the one of Watts and Strogatz.
Then, we present inference results for that network obtained by our algorithm.

First, we review necessary background.
Then, we describe our graph model.
Finally, we present the computed results.

\subsubsection*{Markov Logic Networks}
Markov Logic Networks (abbreviated MLNs) \citep{MLNs} are a popular model from the area of statistical relational learning.
An MLN $\Phi$ is a set of weighted quantifier-free first-order logic formulas with weights taking on values from the real domain or infinity:
$$\Phi = \Set{(w_1, \alpha_1), (w_2, \alpha_2), \ldots, (w_k, \alpha_k)}$$
Given a domain $\Delta$, the MLN defines a probability distribution over possible worlds such as
\begin{align*}
    Pr_{\Phi, \Delta}(\omega) = \frac{\llbracket \omega \models \Phi_{\infty} \rrbracket}{Z} \exp\left(\sum_{(w_i, \alpha_i) \in \Phi_{\Real}} w_i \cdot N(\alpha_i, \omega) \right)
\end{align*}
where $\Phi_{\Real}$ denote the real-valued (soft) and $\Phi_{\infty}$ the $\infty$-valued (hard) formulas, $\llbracket\cdot\rrbracket$ is the indicator function, $Z$ is the normalization constant ensuring valid probability values and $N(\alpha_i, \omega)$ is the number of substitutions to $\alpha_i$ that produce a grounding satisfied in $\omega$.
The distribution formula is equivalent to the one of a Markov Random Field \citep{PGMs}.
Hence, an MLN along with a domain define a probabilistic graphical model and
inference in the MLN is thus inference over that model.

Inference (and also learning) in MLNs is reducible to \wfomc{} \citep{Lifting/Skolemization}.
For each $(w_i, \alpha_i(\Mat{x}_i)) \in \Phi_\Real$, introduce a new formula $\forall \Mat{x}_i: \xi_i(\Mat{x}_i) \Leftrightarrow \alpha_i(\Mat{x}_i)$, where $\xi_i$ is a fresh predicate, $w(\xi_i) = \exp(w_i), \overline{w}(\xi_i) = 1$ and $w(Q) =  \overline{w}(Q) = 1$ for all other predicates $Q$.
Hard formulas are added to the theory as additional constraints.
Denoting the new theory by $\Gamma$ and a query by $\phi$, we can compute the inference as
\begin{align*}
    Pr_{\Phi, \Delta}(\phi) = \frac{\wfomc(\Gamma \wedge \phi, |\Delta|, w, \barw)}{\wfomc(\Gamma, |\Delta|, w, \barw)}.
\end{align*}

\subsubsection*{Watts-Strogatz Model}
The model of Watts and Strogatz \citep{Watts-Strogatz} is a procedure for generating a random graph of specific properties.

First, having $n$ ordered nodes, each node is connected to $K$ (assumed to be an even integer) of its closest neighbors by undirected edges (discarding parallel edges).
If the sequence end or beginning are reached, we wrap to the other end.

Second, each edge $(i, j)$ for each node $i$ is \textit{rewired} with probability $\beta$.
Rewiring of $(i, j)$ means that node $k$ is chosen at random and the edge is changed to $(i,k)$.

\subsubsection*{Our Model}
We start constructing our graph model in the same manner as Watts and Strogatz, with $K=2$.
Ergo, we obtain one cyclic chain going over all our domain elements:
\begin{center}
\begin{tikzpicture}
	\node (1) at (-2,0) [shape=circle,draw] {1};
	\node (2) at (-1,0) [shape=circle,draw] {2}
		edge [-] (1);
 	\node (3) at ( 0,0) [shape=circle,draw] {3}
		edge [-] (2);
	\node (4) at ( 1,0) [shape=circle,draw] {4}
		edge [-] (3);
	\node (dots) at ( 2,0) {\ldots}
		edge [-] (4);
	\node (5) at ( 3,0) [shape=circle,draw] {n}
		edge [-] (dots)
		edge [-, bend left=30] (1);
\end{tikzpicture}
\end{center}

However, we do not perform the rewiring.
Instead, we simply add $m$ additional edges at random.
Hence, all nodes will be connected by the chain and, moreover, there will be various \textit{shortcuts} as well.

Finally, we add a weighted formula saying that \textit{friends} (friendship is represented by the edges) of smokers also smoke.
Intuitively, for large enough weight, our model should prefer those possible worlds where either nobody smokes or everybody does.

Let us now formally state the MLN that we work with:
\begin{align}
    \Phi = \{%
        &(\infty, \neg Perm(x, x)),\label{eq:mln:1}\\
        &(\infty, \neg Perm(x, y) \vee S_1(x)),\label{eq:mln:2}\\
        &(\infty, \neg Perm(x, y) \vee S_2(x)),\label{eq:mln:3}\\
        &(\infty, Pred(x, y) \Rightarrow Perm(x, y)),\label{eq:mln:4}\\
        &(\infty, Pred(x, y) \Rightarrow (x \leq y)),\label{eq:mln:5}\\
        &(\infty, |Perm| = n),\label{eq:mln:6}\\
        &(\infty, |Pred| = n - 1),\label{eq:mln:7}\\
        &(\infty, Perm(x, y) \Rightarrow E(x, y)),\label{eq:mln:8}\\
        &(\infty, E(x, y) \Rightarrow E(y, x)),\label{eq:mln:9}\\
        &(\infty, \neg E(x, y)),\label{eq:mln:10}\\
        &(\infty, |E| = 2n + 2m),\label{eq:mln:11}\\
        &(\ln w, Sm(x) \wedge E(x, y) \Rightarrow Sm(y))\label{eq:mln:12} \}
\end{align}
Senteces \ref{eq:mln:1} through \ref{eq:mln:7} have already been mentioned in the predecessor definition.
They define the basic cyclic chain, albeit a directed one.
Formula \ref{eq:mln:8} copies all $Perm/2$ transitions to $E/2$ and \ref{eq:mln:9} makes the edges undirected.
Moreover, sentence \ref{eq:mln:10} prohibits loops.
Sentence \ref{eq:mln:11} then requires that there are $n + m$ undirected edges in the graph.
As all these are hard constraints, every model must define our predefined graph model.

The only soft constraint is sentence \ref{eq:mln:12}.
By manipulating its weight, we may determine how important it is for the formula to be satisfied in an interpretation.

\subsubsection*{Inference}
We can use \iwfomc{} to run exact inference in the MLN described above.
We may query the probability that a particular domain member (element) smokes.
Obviously, the probability will be the same for any domain member.
We will thus combine all of these together and query for the probability of there being exactly $k$ smokers, instead.

Denote $\Gamma$ the theory obtained when we reduce the MLN $\Phi$ to \wfomc.
We may answer the query as
\begin{align*}
    Pr(|Sm|=k) = \frac{\wfomc(\Gamma \wedge (|Sm|=k), n, w, \barw)}{\wfomc(\Gamma, n, w, \barw)}.
\end{align*}

To relate our model to others which can be modelled without the linear order axiom, we compare the results to inference over a completely random undirected graph with the same number of edges.
Intuitively, completely random graph may form more disconnected components, thus not necessarily preferring the extremes, i.e., either nobody smokes or everybody does.
We also keep the parameter $m$ relatively small since, for large $m$, even the random graph would likely form just one connected component.
The MLN over a random graph is defined as follows:
\begin{align*}
    \Phi' = \{%
        &(\infty, E(x, y) \Rightarrow E(y, x)),\\
        &(\infty, \neg E(x, y)),\\
        &(\infty, |E| = 2n + 2m),\\
        &(\ln w, Sm(x) \wedge E(x, y) \Rightarrow Sm(y)) \}
\end{align*}

Figures \ref{fig:hist:m-5}, \ref{fig:hist:m-8} and \ref{fig:hist:m-10} depict the inference results for a domain size $n=10$ and various weights $w$.
The parameter $m$ is set to $\frac{n}{2}$, $\lceil\frac{3}{4}n\rceil$ and $n$, respectively.
As one can observe, for smaller $w$, our model approaches the binomial distribution just as the random graph model does.
With increasing $w$, the preference for extremes increases as well, and it does so in both models.
However, our model clearly prefers the extreme values more, which is consistent with our intuition above.

\section{Conclusion}
We showed how to compute \wfomc{} in \ctwo{} with linear order axiom in time polynomial in the domain size.
Hence, we showed the language of \ctwo{} extended by a linear order to be domain-liftable.
The computation can be performed using our new algorithm, \iwfomc{}.

\section*{Acknowledgements}
This work was supported by Czech Science Foundation project ``Generative Relational Models’' (20-19104Y) and partially by the OP VVV project {\it CZ.02.1.01/0.0/0.0/16\_019/0000765} ``Research Center for Informatics’’. JT’s work was also supported by a donation from X-Order Lab.

\begin{figure}
     \centering
     \begin{subfigure}[b]{0.49\linewidth}
         \centering
         \includegraphics[width=\textwidth]{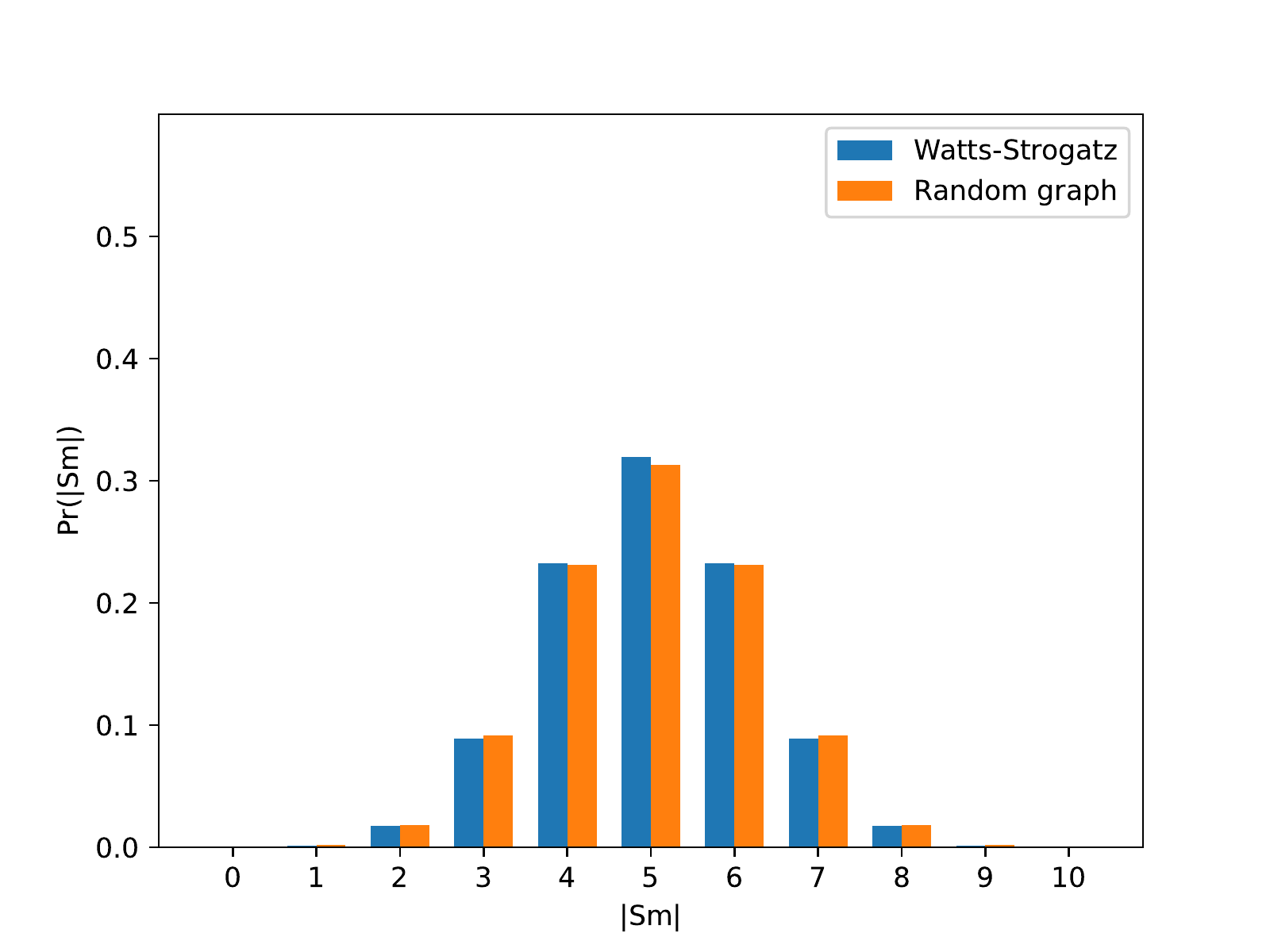}
         \caption{$w = \ln 2$}
         \label{fig:hist:w-ln2-m-5}
     \end{subfigure}
     \begin{subfigure}[b]{0.49\linewidth}
         \centering
         \includegraphics[width=\textwidth]{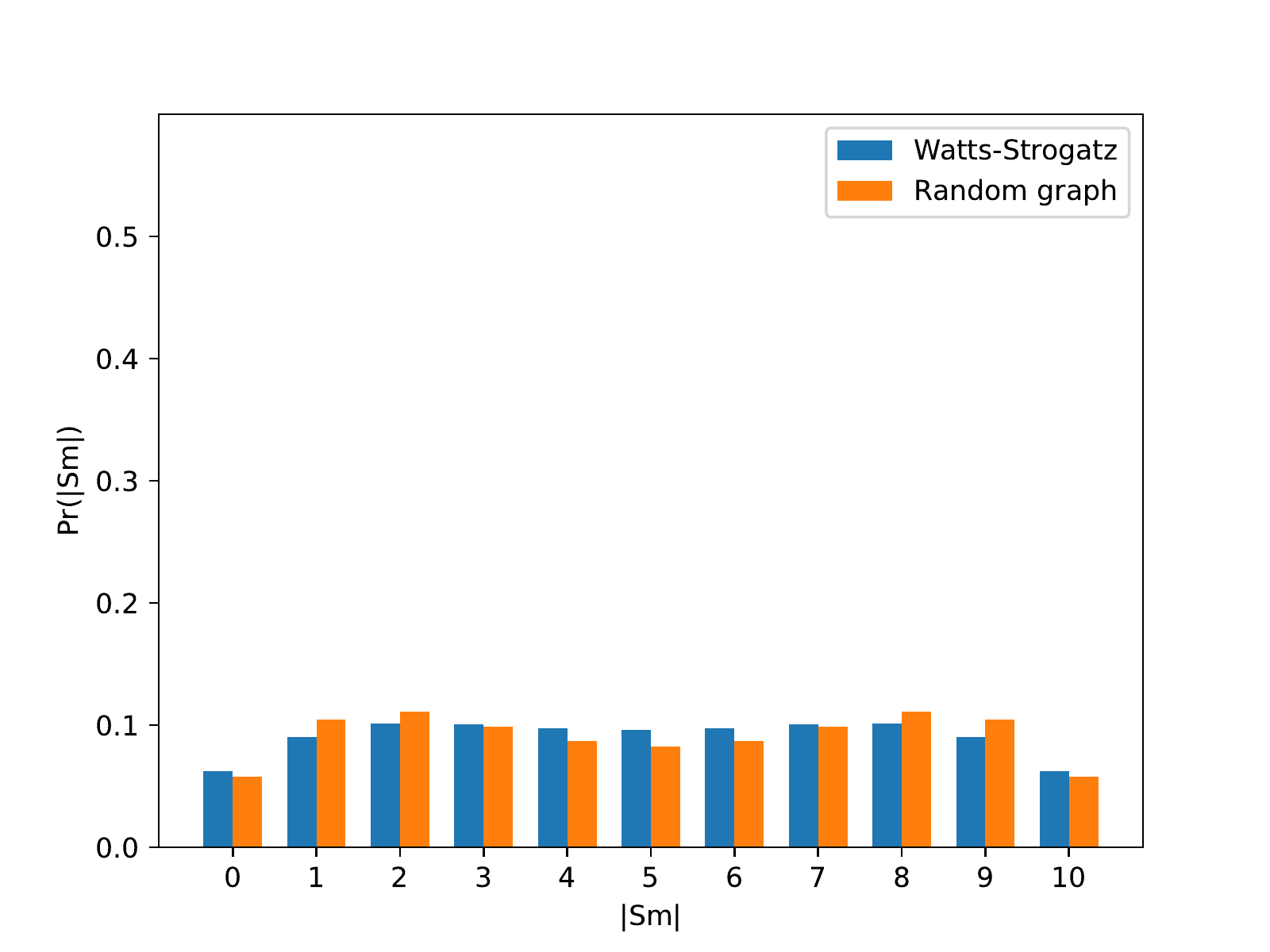}
         \caption{$w = 2$}
         \label{fig:hist:w-2-m-5}
     \end{subfigure}
     
     \begin{subfigure}[b]{0.49\linewidth}
         \centering
         \includegraphics[width=\textwidth]{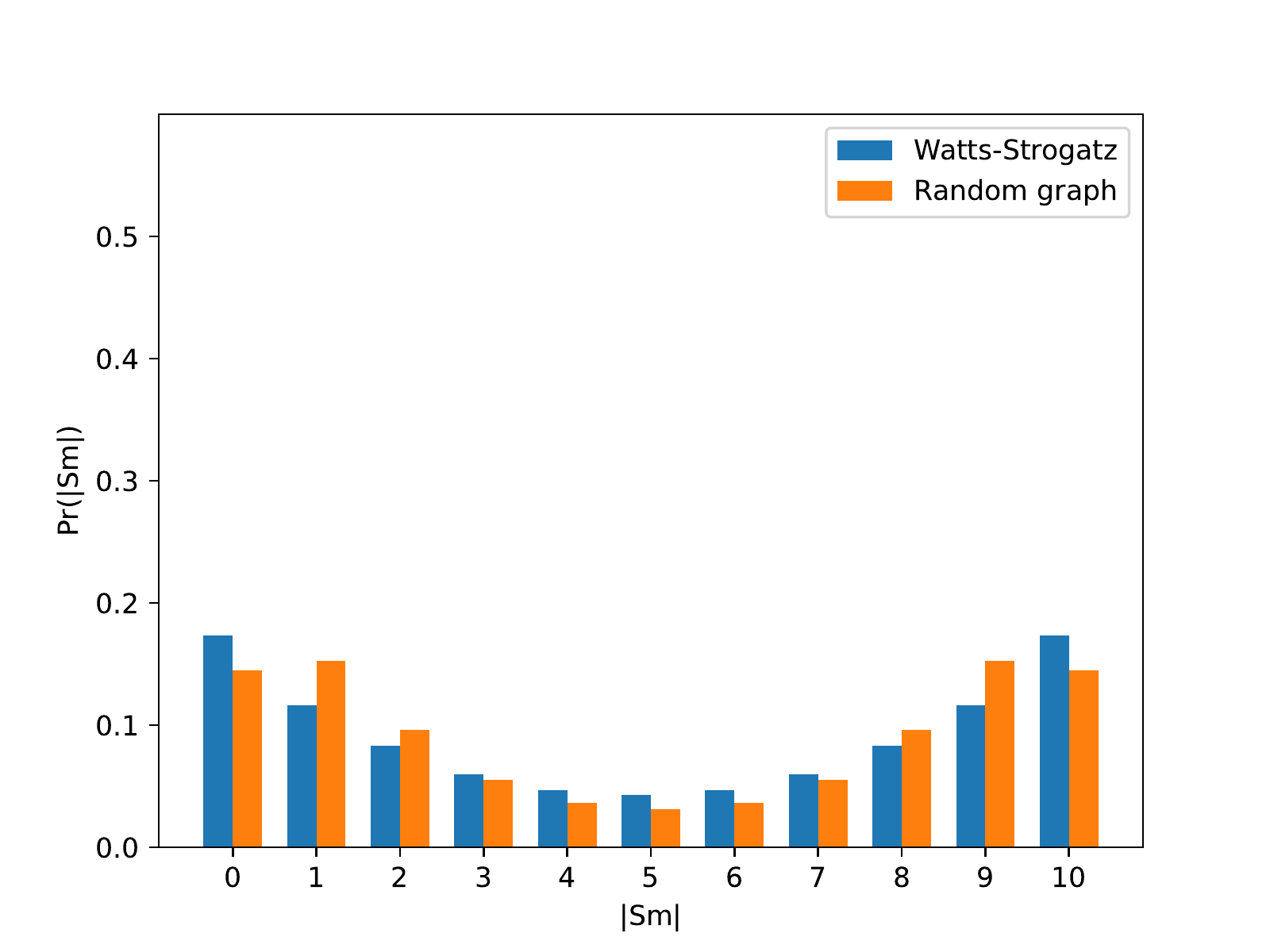}
         \caption{$w = e$}
         \label{fig:hist:w-e-m-5}
     \end{subfigure}
     \begin{subfigure}[b]{0.49\linewidth}
         \centering
         \includegraphics[width=\textwidth]{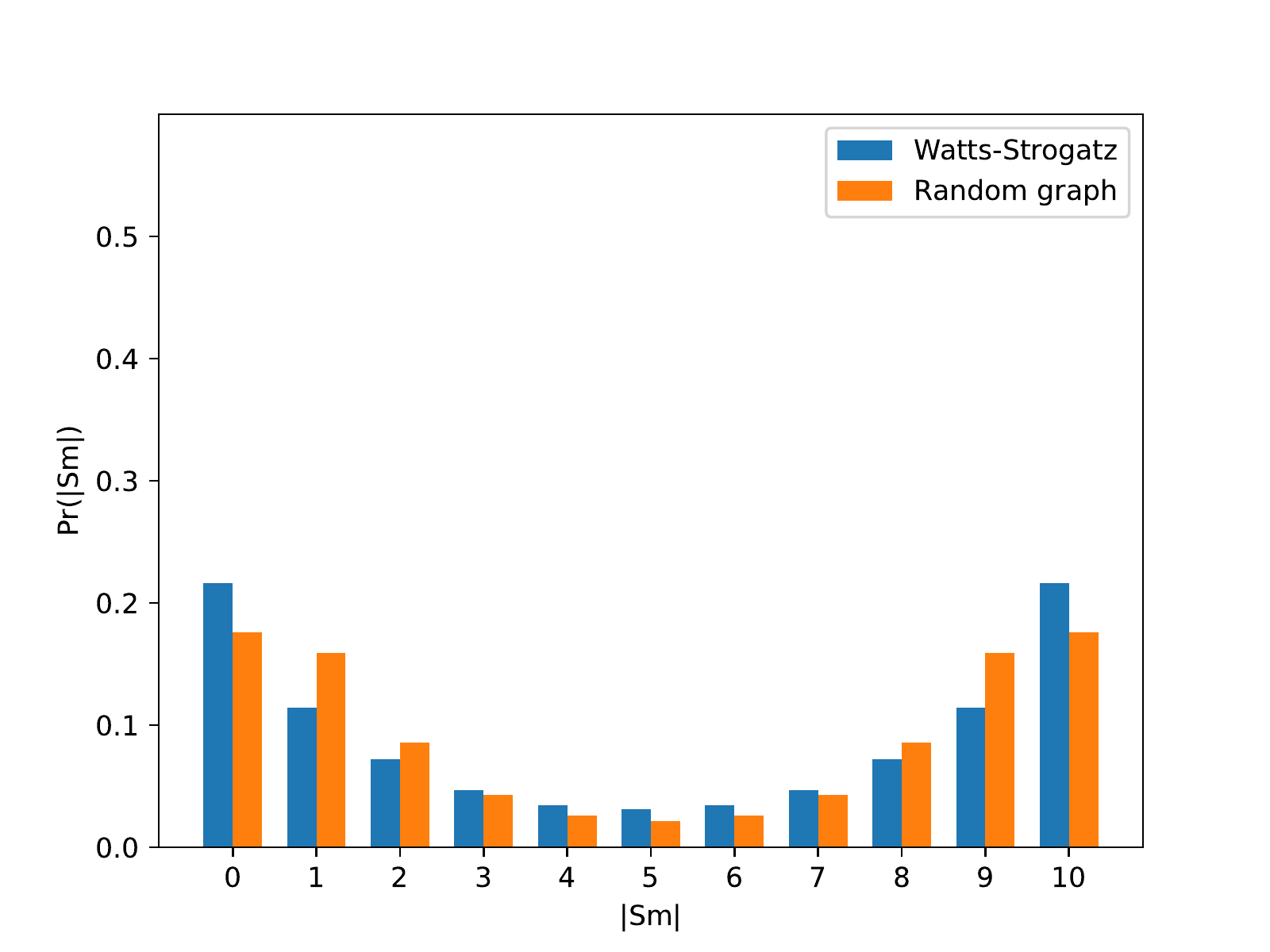}
         \caption{$w = 3$}
         \label{fig:hist:w-3-m-5}
     \end{subfigure}
    \caption{Probability of $n$ smokers for $m = 5$}
    \label{fig:hist:m-5}
\end{figure}

\begin{figure}
     \centering
     \begin{subfigure}[b]{0.49\linewidth}
         \centering
         \includegraphics[width=\textwidth]{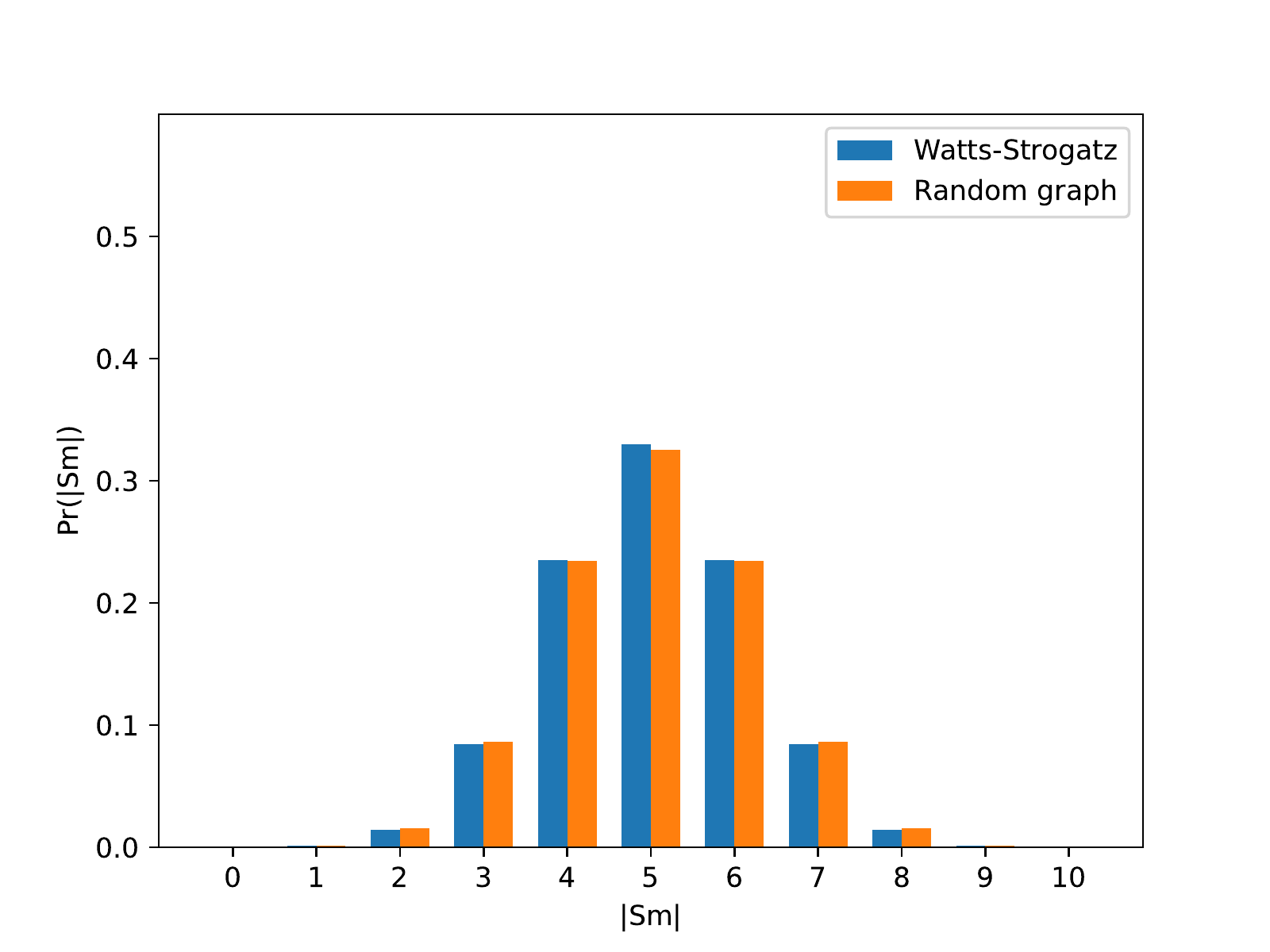}
         \caption{$w = \ln 2$}
         \label{fig:hist:w-ln2-m-8}
     \end{subfigure}
     \hfil
     \begin{subfigure}[b]{0.49\linewidth}
         \centering
         \includegraphics[width=\textwidth]{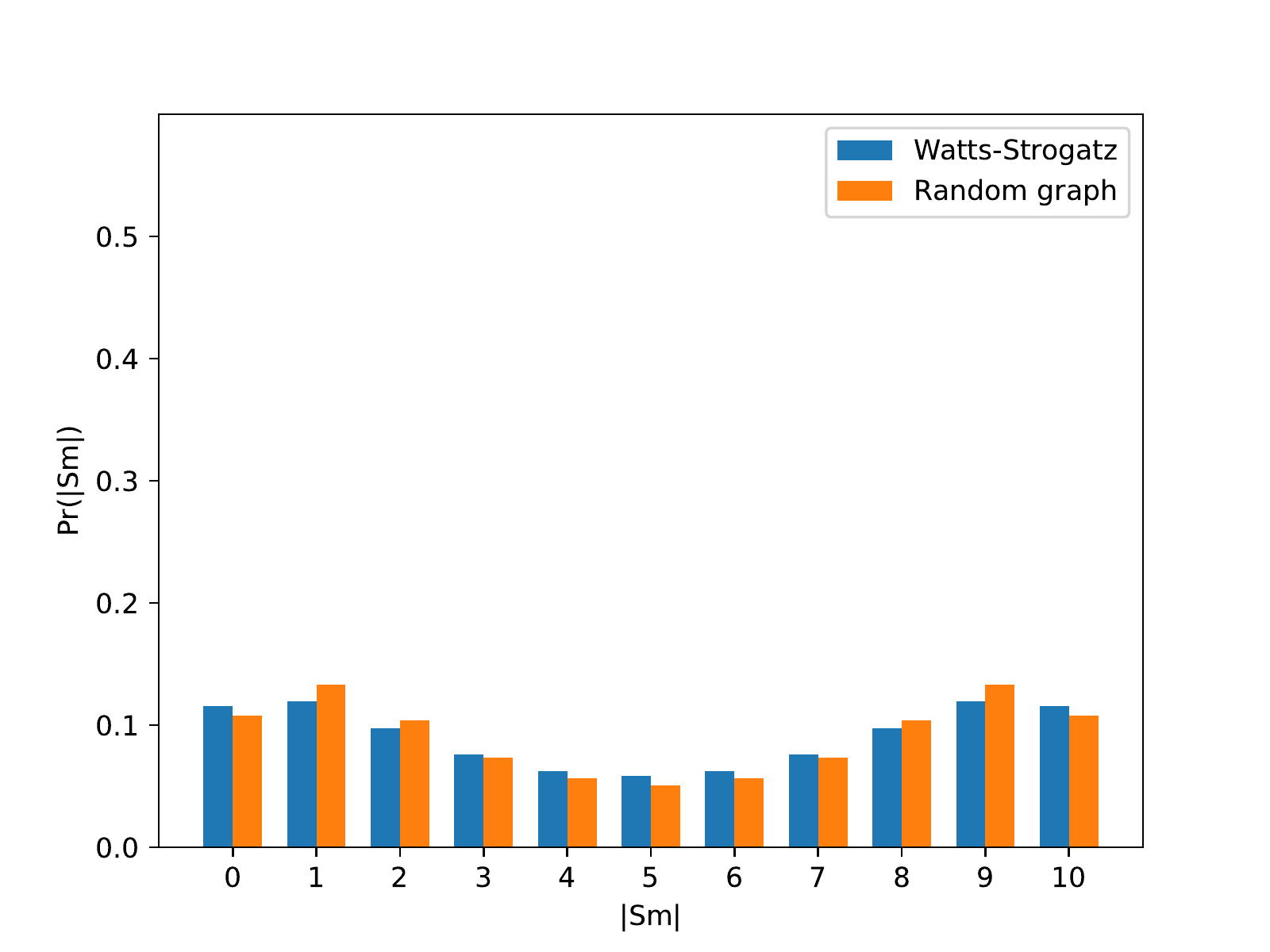}
         \caption{$w = 2$}
         \label{fig:hist:w-2-m-8}
     \end{subfigure}
     \hfil
     \begin{subfigure}[b]{0.49\linewidth}
         \centering
         \includegraphics[width=\textwidth]{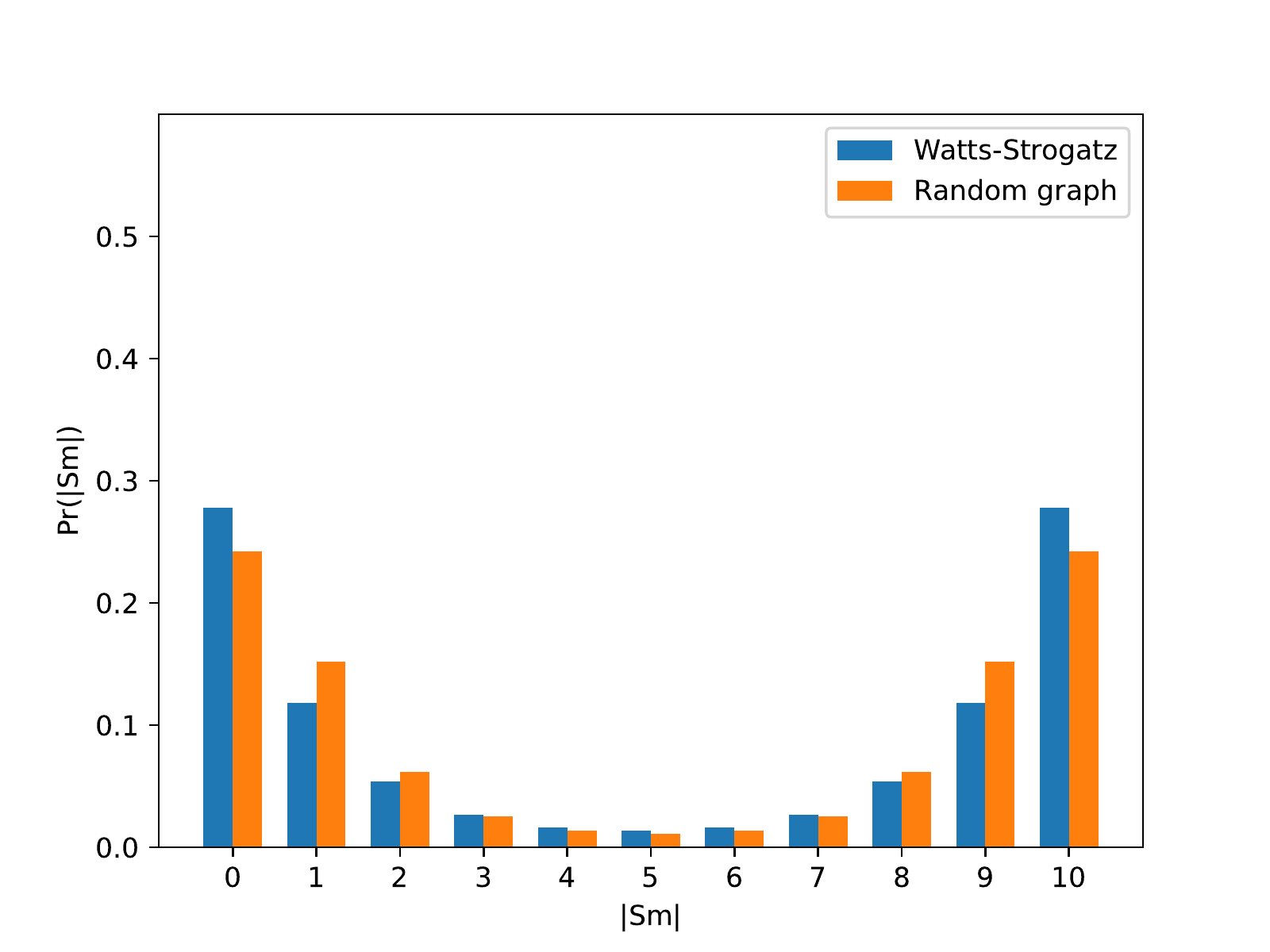}
         \caption{$w = e$}
         \label{fig:hist:w-e-m-8}
     \end{subfigure}
     \hfil
     \begin{subfigure}[b]{0.49\linewidth}
         \centering
         \includegraphics[width=\textwidth]{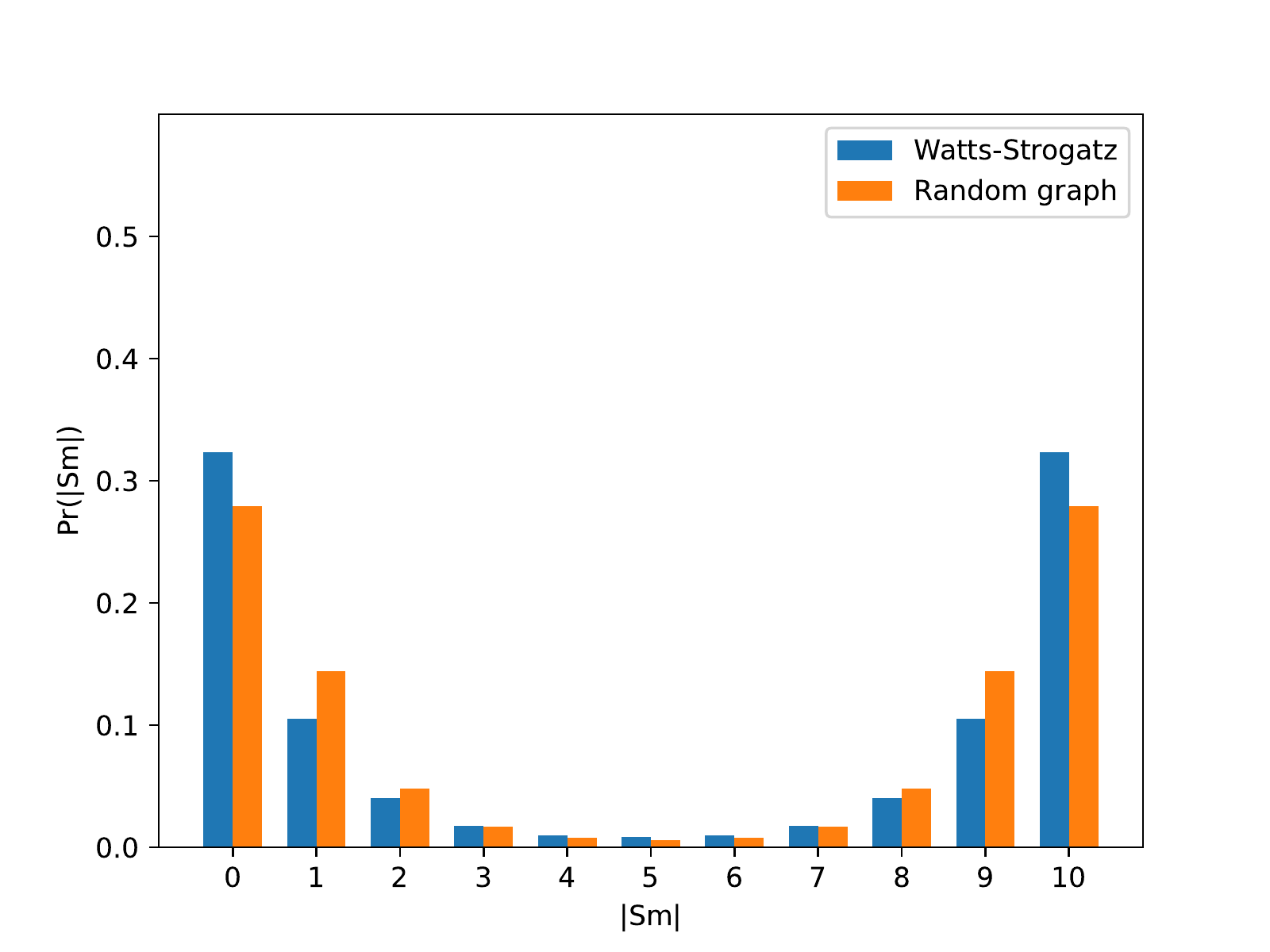}
         \caption{$w = 3$}
         \label{fig:hist:w-3-m-8}
     \end{subfigure}
    \caption{Probability of $n$ smokers for $m = 8$}
    \label{fig:hist:m-8}
\end{figure}

\begin{figure}
     \centering
     \begin{subfigure}[b]{0.49\linewidth}
         \centering
         \includegraphics[width=\textwidth]{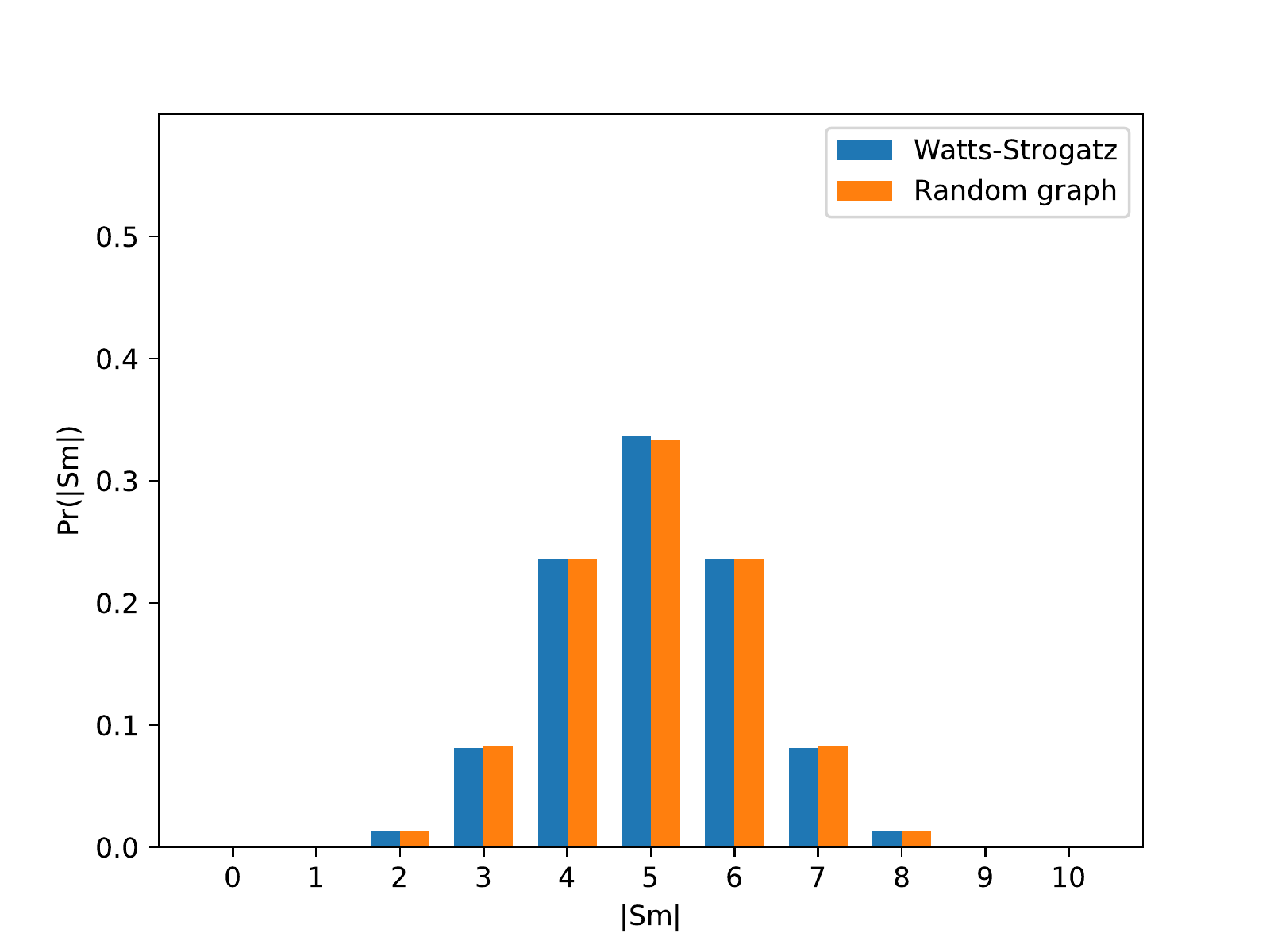}
         \caption{$w = \ln 2$}
         \label{fig:hist:w-ln2-m-10}
     \end{subfigure}
     \begin{subfigure}[b]{0.49\linewidth}
         \centering
         \includegraphics[width=\textwidth]{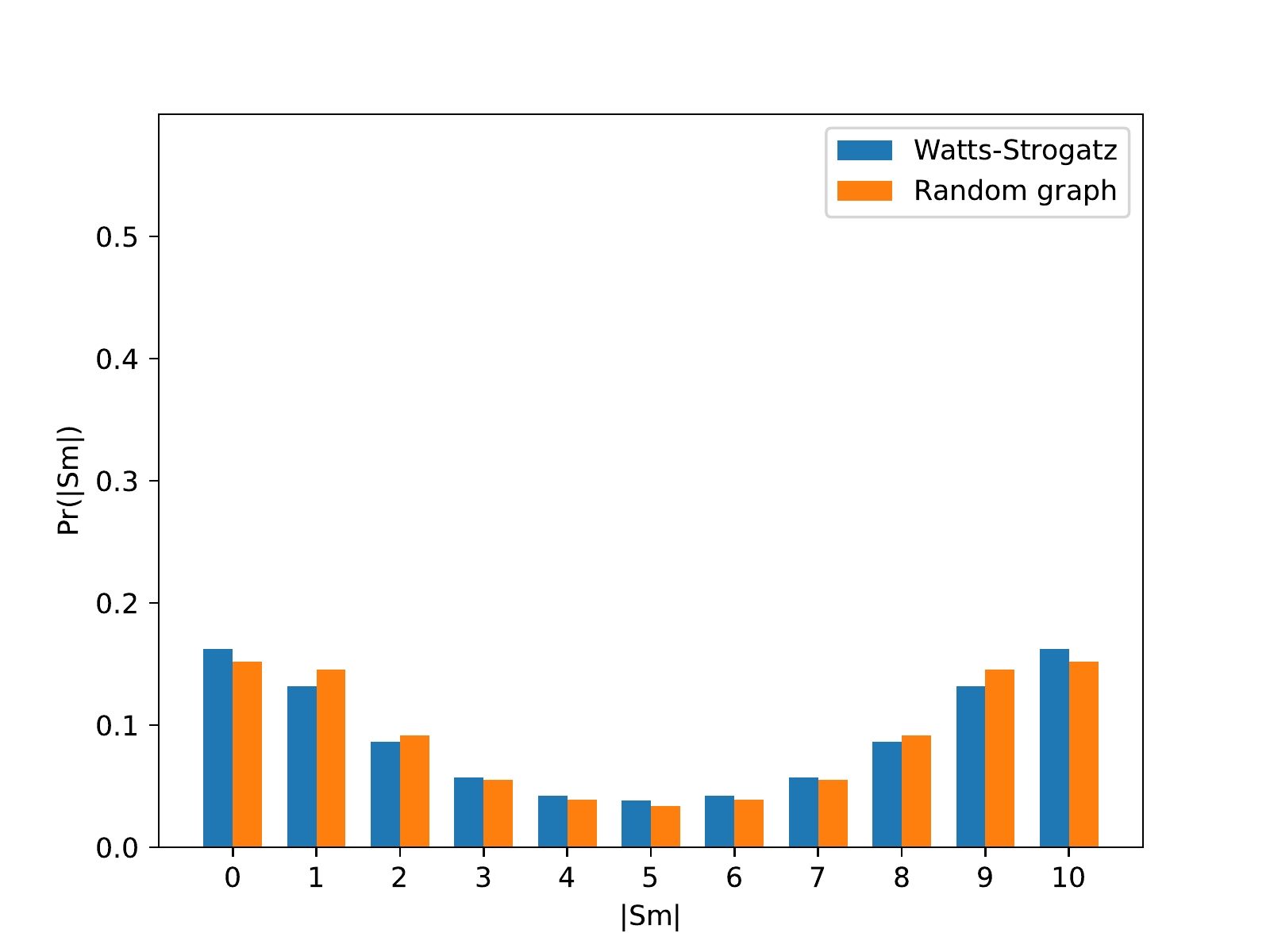}
         \caption{$w = 2$}
         \label{fig:hist:w-2-m-10}
     \end{subfigure}
     
     \begin{subfigure}[b]{0.49\linewidth}
         \centering
         \includegraphics[width=\textwidth]{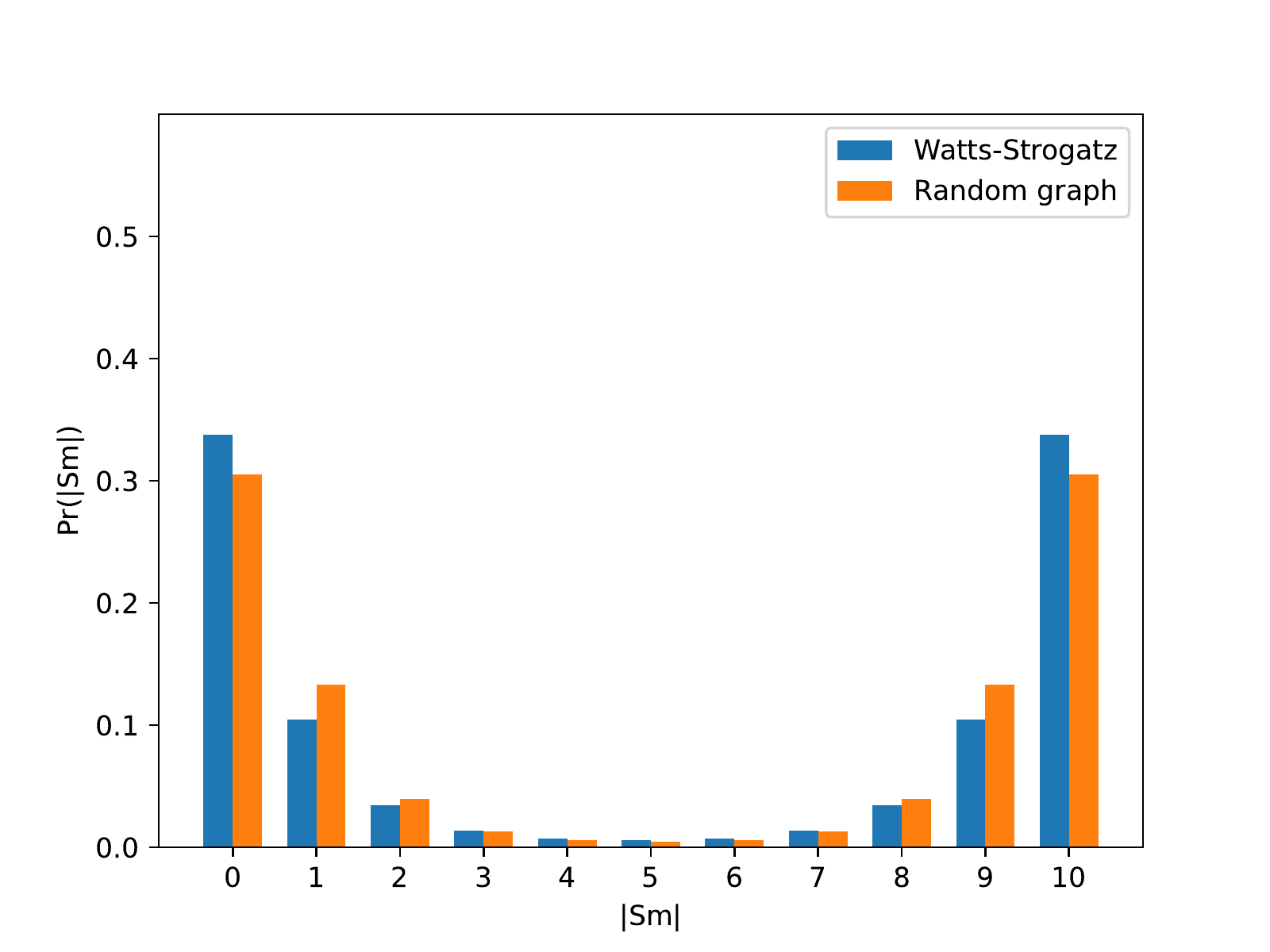}
         \caption{$w = e$}
         \label{fig:hist:w-e-m-10}
     \end{subfigure}
     \begin{subfigure}[b]{0.49\linewidth}
         \centering
         \includegraphics[width=\textwidth]{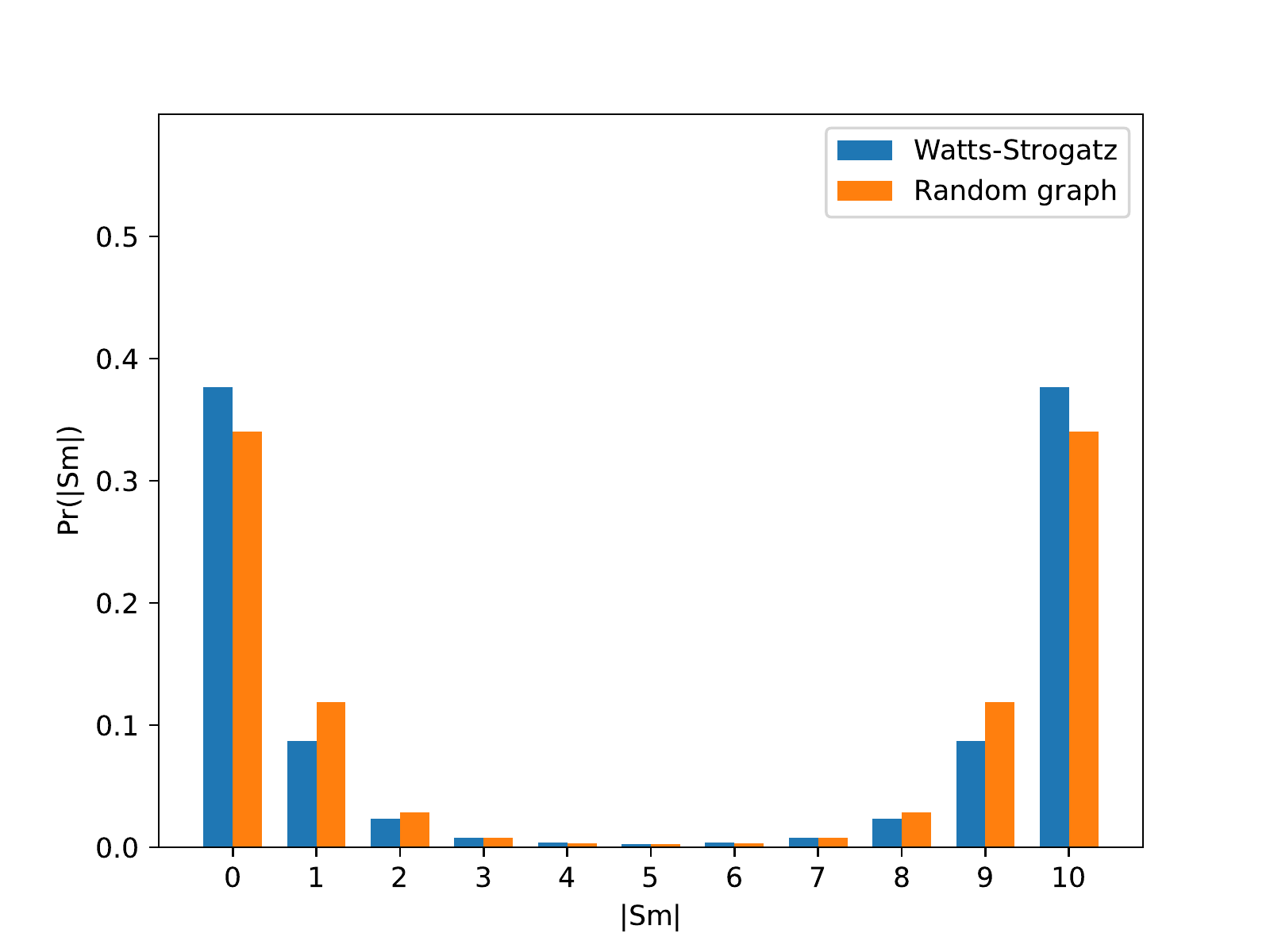}
         \caption{$w = 3$}
         \label{fig:hist:w-3-m-10}
     \end{subfigure}
    \caption{Probability of $n$ smokers for $m = 10$}
    \label{fig:hist:m-10}
\end{figure}

\newpage
\begin{appendices}
\section{Performance Measurements}
\label{app:times}
As is already stated above, we implemented \iwfomc{} in the Julia programming language.
Although our implementation is straightforward and without any further optimizations, measuring its execution times still provides us with an intuition about how the algorithm scales to larger domains that are omnipresent in real-world applications.
Figure \ref{fig:times} depicts the running times of \iwfomc{} on a few problems averaged over multiple executions.
All experiments were performed in a single thread on a computer with a 64-core AMD EPYC 7742 CPU running at speeds 2.25GHz and 512 GB of RAM.

Figure \ref{fig:seq} shows execution times for \textit{head and tail} and \textit{head, middle, tail} examples.
Figure \ref{fig:pred} depicts the running times on the formula $\phi = Linear(\leq, Pred)$, i.e., only finding the number of possible predecessor relations (of which there are $n!$ -- one for each domain ordering).

Finally, Figure \ref{fig:ws} depicts execution times of inference on our Watts-Strogatz-like model averaged over various values of $m$.
To compute the inference, we resorted to one more implementation trick.
Instead of repeatedly computing the probability for each $k \in \Set{0, 1, \ldots, n}$, we turned $w(Sm)$ into a symbolic weight.
Thus, we obtained a polynomial in $w(Sm)$ from the computation of $\wfomc(\Gamma, n, w, \barw)$.
The coefficient for each term of degree $k$ then corresponded to the unnormalized probability of $(|Sm|=k)$.
Hence, we were able to compute the entire probability distribution in one call to \iwfomc{}.
The figure depicts running times for those \textit{symbolic calls}.

\begin{figure}[h]
    \centering
    \begin{subfigure}[b]{0.49\linewidth}
        \centering
        \includegraphics[width=0.9\columnwidth]{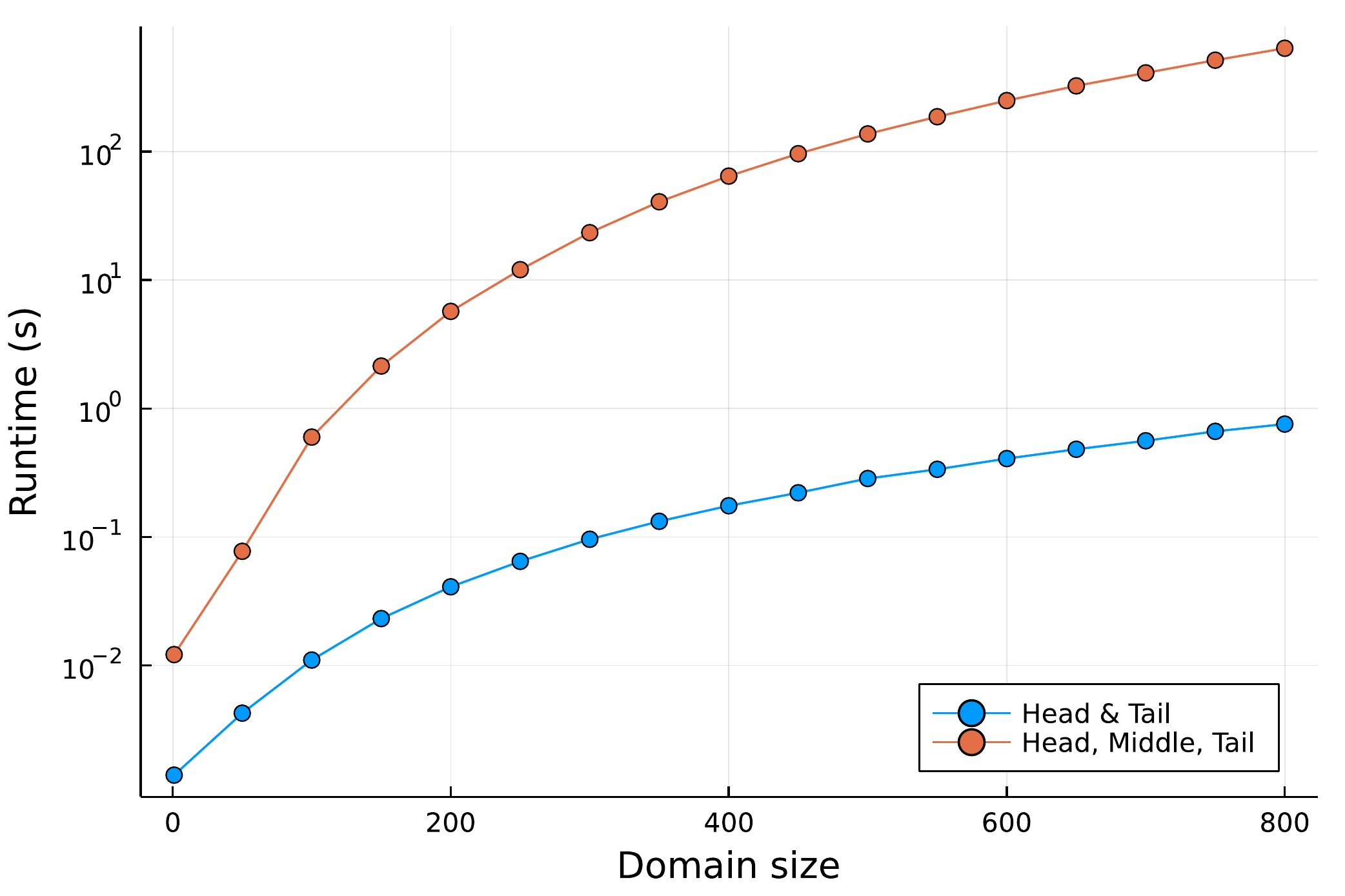}
        \caption{Runtime for counting sequence splits}
        \label{fig:seq}
    \end{subfigure}
     
    \begin{subfigure}[b]{0.49\linewidth}
        \centering
        \includegraphics[width=0.95\columnwidth]{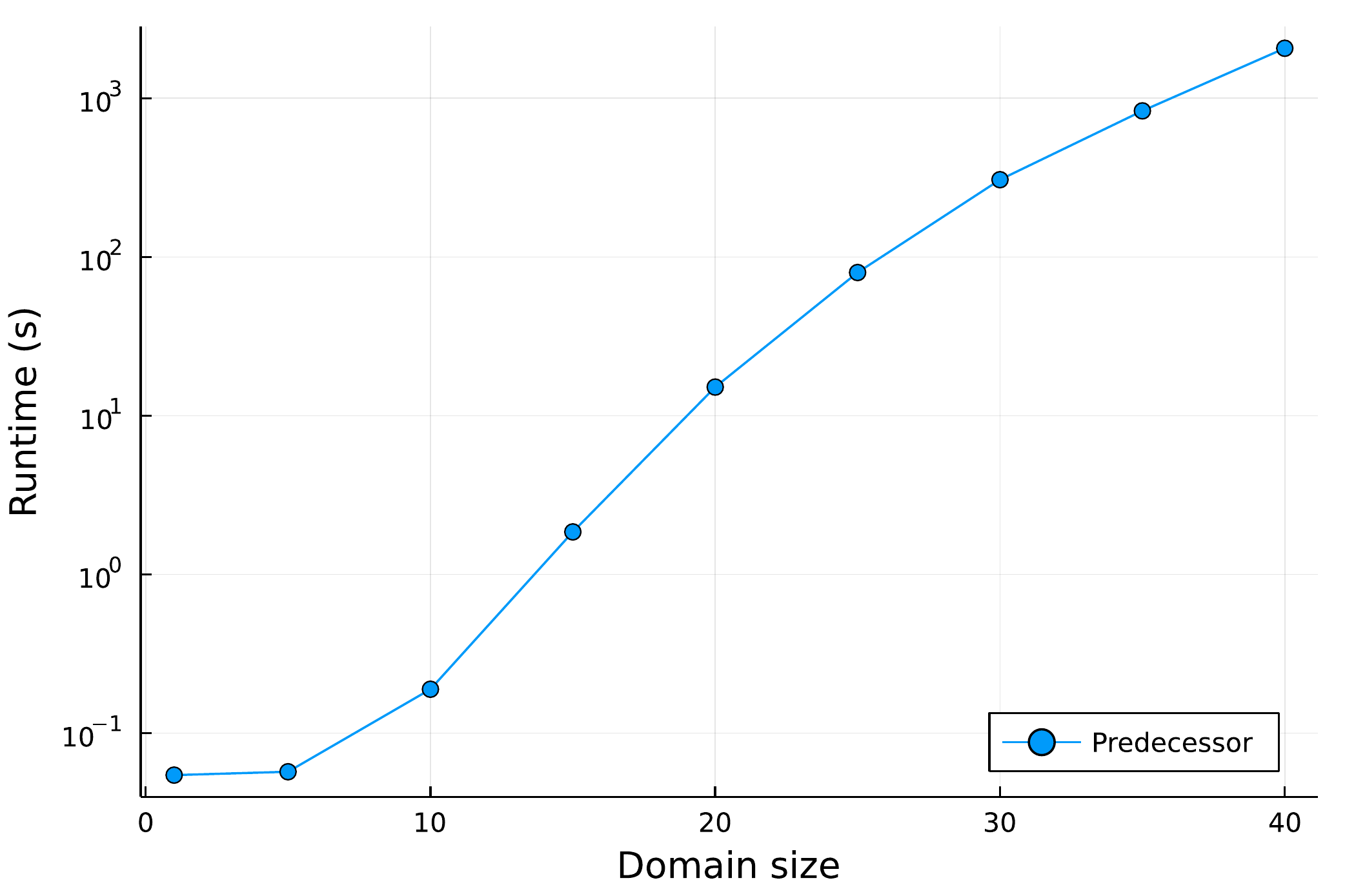}
        \caption{Runtime on the predecessor relation}
        \label{fig:pred}
    \end{subfigure}
    \begin{subfigure}[b]{0.49\linewidth}
        \centering
        \includegraphics[width=\textwidth]{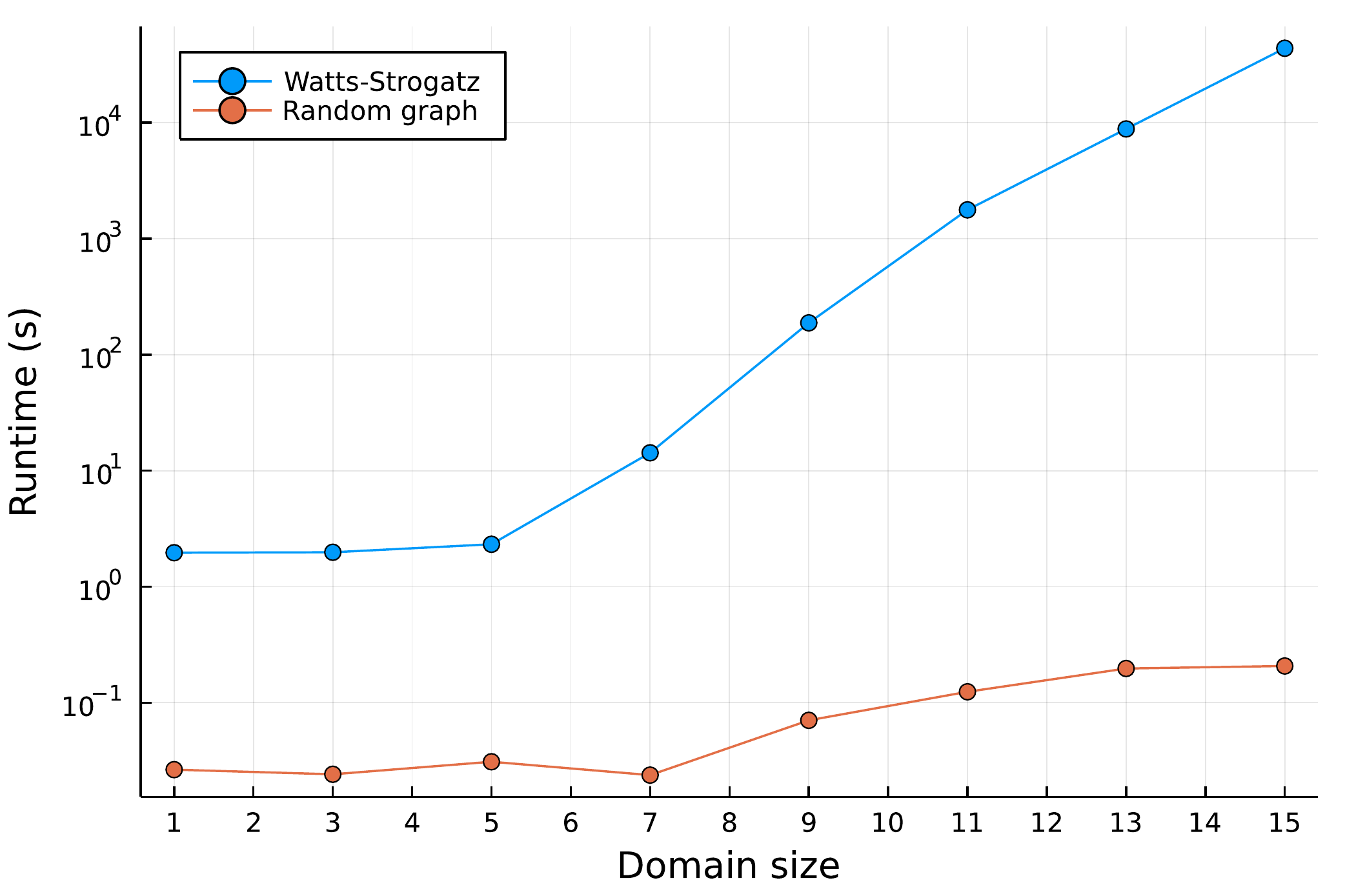}
        \caption{MLN inference runtime}
        \label{fig:ws}
    \end{subfigure}
    \caption{Execution times}
    \label{fig:times}
\end{figure}

\end{appendices}

\newpage
\bibliographystyle{named}
\bibliography{kr}

\end{document}